\documentclass{article}

\usepackage{tikz}
\usetikzlibrary{positioning}

\usepackage{microtype}
\usepackage{graphicx}
\usepackage{subfigure}
\usepackage{booktabs} %
\usepackage[colorlinks=true,linkcolor=blue]{hyperref}

\usepackage{enumitem}
\usepackage[olditem,oldenum]{paralist}

\usepackage{natbib}
\usepackage{float}

\usepackage{amsmath}
\usepackage{amssymb}
\usepackage{mathtools}
\usepackage{amsthm}
\usepackage{mathrsfs}
\usepackage{bbm}
\usepackage{algorithm}
\usepackage{algpseudocode}

\usepackage[capitalize,noabbrev]{cleveref}

\theoremstyle{plain}
\newtheorem{theorem}{Theorem}[section]

\newtheorem{lemma}[theorem]{Lemma}

\theoremstyle{definition}
\newtheorem{definition}[theorem]{Definition}

\theoremstyle{remark}

\usepackage{xspace}
\usepackage{xcolor}
\usepackage{pythonhighlight}
\usepackage{listings}

\definecolor{codegreen}{rgb}{0,0.6,0}
\definecolor{codegray}{rgb}{0.5,0.5,0.5}
\definecolor{codepurple}{rgb}{0.58,0,0.82}
\definecolor{backcolour}{rgb}{0.95,0.95,0.92}

\lstdefinestyle{mystyle}{
    backgroundcolor=\color{backcolour},   
    commentstyle=\color{codegreen},
    keywordstyle=\color{magenta},
    numberstyle=\tiny\color{codegray},
    stringstyle=\color{codepurple},
    basicstyle=\ttfamily\footnotesize,
    breakatwhitespace=false,         
    breaklines=true,                 
    captionpos=b,                    
    keepspaces=true,                 
    numbers=left,                    
    numbersep=5pt,                  
    showspaces=false,                
    showstringspaces=false,
    showtabs=false,                  
    tabsize=2
}

\usepackage{amsmath}
\usepackage{booktabs}
\usepackage{multirow}
\usepackage{wrapfig}

\usepackage{comment}
\usepackage[thinc]{esdiff}
\usepackage{graphicx}

\newcommand{\edo}[1]{{\color{blue}#1}}
\definecolor{bostonuniversityred}{rgb}{0.8, 0.0, 0.0}

\definecolor{ao}{rgb}{0.0, 0.5, 0.0}

\definecolor{bondiblue}{rgb}{0.0, 0.6235294117647059, 0.6941176470588235}
\definecolor{pantonegreen}{rgb}{0.0, 0.6627450980392157, 0.2784313725490196}

\definecolor{oceanblueboat}{rgb}{0.0, 0.4823529411764706, 0.7686274509803922}

\usepackage{adjustbox}

\newcommand*{\rom}[1]{\expandafter\@slowromancap\romannumeral #1@}

\newcommand{\editornote}[2]{{\color[rgb]{#1}#2}}
\newcommand{\NDY}[1]{\editornote{0.0,0.8,0.4}{[NdY: #1]}}
\newcommand{\TODO}[1]{\editornote{1.0,0.0,0.0}{[TODO: #1]}}
\newcommand{\todo}[1]{\editornote{1.0,0.0,0.0}{[TODO: #1]}}
\renewcommand{\TODO}[1]{}
\renewcommand{\todo}[1]{}
\renewcommand{\NDY}[1]{}
\renewcommand{\edo}[1]{}

\usepackage{tikz}

\newcommand{\from}{\colon} %

\def\R{{\mathbb{R}}}

\let\oldPr\Pr
\renewcommand{\Pr}{\oldPr\nolimits}
\newcommand{\E}{\mathbb{E}}
\newcommand{\transp}[1]{#1^{\!\top}\!}
\DeclareMathOperator*{\argmax}{arg\,max}

\newcommand{\deq}{\mathrel{\mathop{:}}=}

\def\d{\operatorname{d}\!{}}
\newcommand{\calB}{\mathcal{B}}
\DeclareMathOperator{\Id}{Id}

\newcommand{\norm}[1]{\left\lVert#1\right\rVert}

\newcommand{\1}{\mathbbm{1}}

\title{Finer Behavioral Foundation Models via Auto-Regressive Features
and Advantage Weighting}
\author{Edoardo Cetin, Ahmed Touati, Yann Ollivier}

\begin{document}

\maketitle

\begin{abstract}

The forward-backward representation (FB) is a recently proposed framework
\citep{fb_0shot,fb_v1} to train \emph{behavior foundation models} (BFMs)
that aim at providing zero-shot efficient policies
for any new task specified in a given reinforcement learning (RL)
environment, without training for each new task. Here we address two core
limitations of FB model training.

First, FB, like all successor-feature-based methods, relies on a
\emph{linear} encoding of tasks: at test time, each new reward function
is linearly projected onto a fixed set of pre-trained features. This
limits expressivity as well as precision of the task representation.  We
break the linearity limitation by introducing \emph{auto-regressive
features} for FB, which let fine-grained task features depend
on coarser-grained task information.  This can represent arbitrary
nonlinear task encodings, thus significantly increasing expressivity of
the FB framework.

Second, it is well-known that training RL agents from offline
datasets often requires specific techniques.%
We show that FB works well together with such offline RL
techniques, by adapting techniques from~\citep{nair2020awac, asac} for FB.
This is necessary to get non-flatlining performance in some datasets,
such as DMC Humanoid.

As a result, we produce efficient FB BFMs for a number of new
environments%
.  Notably, in the D4RL
locomotion benchmark, the generic FB agent matches the performance of
standard single-task offline agents (IQL, XQL).  In many setups, the
offline techniques are needed to get any decent performance at all. The
auto-regressive features have a positive but moderate impact,
concentrated on
tasks requiring spatial precision
and task generalization beyond the behaviors represented in the
trainset.

Together, these results establish that generic, reward-free FB BFMs
can be competitive with single-task agents on standard benchmarks, while
suggesting that expressivity of the BFM is not a
key limiting factor in the environments tested.
\end{abstract}

\section{Introduction}

The forward-backward representation (FB) is a recently proposed framework
\citep{fb_0shot,fb_v1} to train behavior foundation models (BFMs) from
offline data. BFMs promise to provide zero-shot efficient policies for
any new task specified in a given reinforcement learning (RL)
environment, beyond the tasks and behaviors in the training set.  This
contrasts with traditional offline RL and imitation learning approaches,
which are trained to accomplish individual target tasks, with no
mechanism to tackle new tasks without repeating the full training
procedure.

The FB approach strives to learn an agent that recovers many possible
behaviors in a given environment, based on learning \emph{successor
measure} representations, without any reward signal. After training, an
FB agent can be prompted via several kinds of task description: an
explicit reward function, a goal state, or even a single demonstration
\citep{fb_imitation}.

However, in its current formulation, FB has been shown effective only for toy problems and relatively simple
locomotion tasks and when trained on undirected
datasets collected
via unsupervised exploration \citep{RND}.

Here, we tackle two core limitations of the ``vanilla'' FB framework,
namely, the difficulty to learn from complex offline datasets, and the
linear correspondence between tasks and features. As a result, we can build
high-performing FB BFMs for a series of new environments. Our main
contributions are the following:

\begin{itemize}
\item  We show that the vanilla FB policy optimization leads to poor
performance when learning from datasets made of a few near-optimal
examples for a few specific tasks. This failure is exactly analogous to
naively using online RL algorithms in the offline setting, a well-studied
problem \citep{bcq, brac, offline_rl_review}. 
This explains the poor performance of vanilla FB on the \emph{D4RL benchmark}, as reported in recent unsupervised RL works~\citep{parkfoundation, frans2024unsupervised}.

\item Accordingly, we introduce a new policy optimization step for FB, to
improve learning from offline datasets demonstrating complex behaviors.
In particular, we integrate an improved version of
\emph{advantage-weighted regression}~\citep{awac}, together with recent
advancements from the offline RL literature~\citep{asac} and additional
algorithmic refinements to FB (Section~\ref{sec:aw}).

We show these changes are crucial to train FB with common offline
datasets beyond pure RND exploration and scale to more challenging
environments, often making the difference between near-zero and
satisfactory performance (Section~\ref{sec:experiments}). %

\item We overcome a core theoretical limitation of FB and, more
generally, of all \emph{successor features}
frameworks~\citep{barreto2017successor, borsa2018universal}: their linear
correspondence between reward functions and task representation vectors.
Indeed, in these frameworks, at test time, the reward function is
linearly projected onto a fixed set of pre-trained features. This results
in ``reward blurring'' and limits spatial precision in the task
representation~\citep{fb_v1}.

We introduce a new \emph{auto-regressive} encoding of task features
(Section~\ref{sec:ar}), that breaks the linearity constraint by letting
fine-grained task features depend on coarser-grained task information.
This allows for universal approximation of any arbitrary task space
(Appendix, Theorem~\ref{thm:univapprox}).

We show that auto-regressive features make a moderate but systematic
difference when learning \emph{new} test tasks far from ones considered
to build the datasets, or for tasks requiring precise goal-reaching (eg,
$15\%$ relative increase for goal-reaching in the Jaco arm environment).

\item
With
these improvements, we show that advantage-weighted
autoregressive FB (FB-AWARE) extends FB performance to new environments
such as Humanoid and the locomotion environments in the canonical \emph{D4RL
benchmark}~\citep{d4rl}. On the latter, FB-AWARE matches the performance
of standard offline RL agents trained on a single task with full access
to rewards (Section~\ref{sec:d4rl}), further vindicating the use of behavior
foundation models for zero-shot RL.

\end{itemize}

\section{Preliminaries
: Notation, Forward-Backward Framework for
Behavioral Foundation Models
}

\paragraph{Markov decision process, notation.}
We consider a reward-free Markov decision process (MDP) $\mathcal{M}=(S,A,P,\gamma)$ 
with state space $S$, action space $A$, transition probabilities
$P(s'|s,a)$ from state $s$ to $s'$ given action $a$,
and discount factor $0 < \gamma < 1$ \citep{sutton2018reinforcement}. 
A policy $\pi$ is a function $\pi\from S\to \mathrm{Prob}(A)$ mapping a
state $s$ to the probabilities of actions in $A$.
Given
$(s_0,a_0)\in S\times A$ and a policy $\pi$,
we denote $\Pr(\cdot|s_0,a_0,\pi)$ and $\E[\cdot|s_0,a_0,\pi]$ the
probabilities and expectations under state-action sequences $(s_t,a_t)_{t
\geq 0}$ starting at $(s_0,a_0)$ and following policy $\pi$ in the
environment, defined by sampling $s_t\sim P( s_t|s_{t-1},a_{t-1})$ and
$a_t\sim \pi( a_t | s_t)$.
Given any reward
function $r\from S \to \R$, the $Q$-function of $\pi$ for $r$ is
$Q_r^\pi(s_0,a_0)\deq \sum_{t\geq 0} \gamma^t \E
[r(s_t)|s_0,a_0,\pi]$. The \emph{value function} of $\pi$ for $r$ is
$V_r^\pi(s)\deq \E_{a\sim \pi(s)} Q_r^\pi(s,a)$, and the
\emph{advantage function} is
is $A_r^\pi(s,a)\deq Q_r^\pi(s,a)-V_r^\pi(s)$.

We assume access to a dataset consisting of \emph{reward-free} observed
transitions $(s_t,a_t,s_{t+1})$ in the environment. We denote by $\rho$
the distribution of states $s_{t+1}$ in the training set.

\paragraph{Behavioral foundation models, zero-shot RL.} A
\emph{behavioral foundation model} for a given reward-free MDP, is an
agent that can produce approximately optimal policies for any reward
function $r$ specified at test time in the environment, without performing
additional learning or fine-tuning for each new reward function. An early
example of such a model includes universal successor features (SFs)~\citep{borsa2018universal}, which depend on a set of (sometimes handcrafted) features:
at test time, the reward is linearly projected onto the features, and a
pre-trained policy is applied. Forward-backward representations (defined
below) are another one, mathematically related to SFs.
\cite{fb_0shot} compares a number of variants of SFs and FB on a number
of empirical problems. \NDY{this is like a mini-related work, but it fits
well here}

\paragraph{The forward-backward framework.} The FB framework
\citep{fb_v1,fb_0shot} is a theoretically and empirically well-supported
way 
to train BFMs,
based on
learning an efficient representation of the \emph{successor measures}
$M^{\pi}$ for various policies $\pi$. For each state-action $(s_0,a_0)\in
S\times A$, this is a measure over states, describing the distribution of
future
states visited by starting at $(s_0,a_0)$ and following policy $\pi$. It
is defined as
\begin{equation}
M^\pi(s_0,a_0,X)\deq \sum_{t\geq 1} \gamma^t \Pr(s_t\in X \mid
s_0,a_0,\pi)
\end{equation}
for any subset $X\subset S$. $M^\pi$ satisfies a measure-valued Bellman
equation \citep{successorstates}, which can be used to learn
approximate parametric models of $M$.

\cite{fb_v1} propose to learn a finite-rank parametric model of $M$, as
follows:
\begin{equation}
\label{eq:FBmodel}
M^{\pi_z}(s_0,a_0,X)\approx \int_{s\in X}
\transp{F(s_0,a_0,z)}B(s)\rho(\d s)
\end{equation}
where $\rho$ is the data distribution, where $F$ and $B$ take values in
$\R^d$, where $z\in \R^d$ is a task encoding vector, and where
\begin{equation}
\label{eq:piz}
\pi_z(s)=\argmax_a \transp{F(s,a,z)}z
\end{equation}
is a parametric policy depending on $z$. $F$, $B$ and $\pi_z$ are learned
at train time.
At test time, given a reward
function $r$, one estimates the task representation vector
\begin{equation}
\label{eq:FBz}
z=\E_{s\sim \rho} [r(s)B(s)]
\end{equation}
and then the policy $\pi_z$ is applied.

The main result of \cite{fb_v1} is that when \eqref{eq:FBmodel}--\eqref{eq:piz}
hold, then for \emph{any} reward function $r$, the policy $\pi_z$ so
obtained is optimal. At test time, reward functions for FB may also be specified
through an expert demonstration \citep{fb_imitation}.

The full algorithm for FB training is provided in Algo.~\ref{alg:fb}
(Appendix~\ref{app: algos}).

\section{Breaking Some Key Limitations of the
Forward-Backward
Framework
}

\subsection{Auto-Regressive Features for Non-Linear Task Encoding} %
\label{sec:ar}

\paragraph{Intuition for auto-regressive FB: nonlinear task encoding.}
Forward-backward (FB) representations and their predecessor, universal
successor features (SFs), attempt to solve zero-shot RL by linearly
projecting new tasks (reward functions $r$) onto a set of features
$B\from S\to \R^d$. At test time, when facing a new reward function $r$,
a task encoding $z\in \R^d$ is computed by
$z=\E [r(s)B(s)]$ (FB) or $z=(\E [\phi(s)\transp{\phi(s)})^{-1} \E
[r(s)\phi(s)]$ (SFs). Then a pretrained policy $\pi_z$ is applied.

FB aims at learning the features $B$ that minimize the error from this
process: $B$ is obtained by a finite-rank approximation of the operator
that sends a reward $r$ to its $Q$-function. Bringing the FB loss to $0$
(which requires infinitely many features) guarantees successful zero-shot
RL for any reward function $r$. Theoretically, the features $B$ in FB
``most linearize'' the computation of $Q$-functions, and empirically this
brings better performance than other feature choices \citep{fb_0shot}.

Still, even with the best features $B$, the task encoding $z$ is
\emph{linear}, because $z=\E [r(s)B(s)]$ is linear in $r$: tasks are
identified by the size-$d$ vector of their correlations with a fixed set
of $d$ pre-trained features $B$.

The standard FB framework learns a rank $d$ approximation
by focusing on the main eigenvectors of the environment dynamics
\citep{successorstates}. Projecting the reward onto these eigenvectors
can remove spatial precision, creating short-term reward blurring
\citep{fb_v1}.

This is clearly suboptimal. Intuitively, if we first acquire information
that the rewards are located in the top-left corner of $S$, we would like
to use more precise features located in the top-left corner to better
identify the reward function there.

\emph{Auto-regressive features} make this possible, while still keeping
most of the theoretical properties of plain FB. The idea is to compute the
task encoding $z=\E[r(s)B(s)]$ progressively, and let the
later-computed features $B$ depend on the early components of $z$. We
decompose $z$ and $B$ into $K$ blocks $z=(z_1,z_2,\ldots,z_K)$ and
$B=(B_1,B_2,\ldots,B_K)$. We first
compute $z_1=\E[r(s)B_1(s)]$ as in plain FB.
But then we compute
$z_2=\E[r(s)B_2(s,z_1)]$ where the second block of features $B_2$ is
allowed to depend on $z_1$, thus conditioning the features on the task
information provided by $z_1$. This can be iterated:
$z_3=\E[r(s)B_3(s,z_1,z_2)]$, etc. The resulting vector
$=(z_1,z_2,\ldots,z_K)$ encodes the task in an auto-regressive manner,
where the meaning of $z_i$ depends on $z_{1:i-1}$.

Intuitively, $z_1$ provides a ``coarse'' task encoding by linear
features. Then we compute a further, finer task encoding
$z_2$ by computing the correlation of $r$ with features $B_2$ that depend
on the coarse task encoding $z_1$. Hopefully the
features $B_2$ can become more specialized and provide a better task
encoding.

In practice, the main change with respect to plain FB training is that
$B$ depends on $z$. 
We now represent the successor measures $M^{\pi_z}$ by
$\transp{F(z)}B(z)$, instead of simply $\transp{F(z)}B$ which shares the same
$B$ for all policies. This allows for a better fit of the FB model.
This
is formalized in the next section and in Appendix~\ref{sec:theory}.

Contrary to plain FB, the task encoding $r\mapsto z$ becomes fully
nonlinear: the set of tasks $r$ represented exactly becomes a nonlinear
submanifold of all possible tasks, instead of a $d$-dimensional subspace.
Even with just two levels of features, this model is able to represent an
arbitrary nonlinear mapping between reward functions $r$ and task
representations $z$ (Appendix, Theorem~\ref{thm:univapprox}), instead of
just linearly projecting the reward onto a fixed basis of features. This
greatly extends the theoretical expressivity of the FB and successor feature
frameworks.

This model also encodes a hierarchical prior on tasks, favoring tasks
that can be described through a cascade of more and more specialized
task features.

\paragraph{FB with auto-regressive task encoding: formal description.} 
Auto-regressive features extend plain FB by letting $B$ depend on $z$. In ordinary
FB, this would be problematic, since the task encoding $z=\E[r(s)B(s)]$ used
at test time becomes a fixed point equation if $B$ depends on $z$.
However, this fixed point equation can be handled easily if $B$ has a
hierarchical or auto-regressive structure.

\begin{definition}
A feature map $B\from S\times \R^d \to \R^d$ is called
\emph{auto-regressive} if, for any $1\leq i \leq d$ and any $(s,z)\in
S\times\R^d$, the
$i$-th component of $B(s,z)$ only depends on $(z_1,\ldots,z_{i-1})$ and
not on $(z_i,\ldots,z_d)$.
\end{definition}

For such models, we can easily compute fixed-points values of the type
$z=B(s,z)$, by first computing the component $B_1$ of the output (which
does not dependent on $z$), which determines $z_1$, which allows us to
compute the component $B_2$ of the output, which determines $z_2$, etc.

In practice, auto-regressive models $B(s,z)$ can be built by splitting
both the representation vector $z\in \R^d$ and the output $B(s,z)\in
\R^d$ into $k$ ``auto-regressive groups'' of dimension $d/k$.  The first
group $B_1(s,z)$ in the output of $B$ is actually independent of $z$, and
the $i$-th group $B_i(s,z)$ of the output of $B$ only takes as inputs the
previous groups $z_1,\ldots,z_{i-1}$ of $z$. At test time, this allows us
to compute the fixed point $z=\E_{s\sim \rho}[r(s) B(s,z)]$ by first
estimating the first group, $z_1=\E_{s\sim \rho} [r(s) B_1(s)]$ similar
to plain FB. Then the other groups are computed iteratively:
$z_{i+1}=\E_{s\sim \rho} [r(s) B_{i+1}(s,z_1,\ldots,z_i)]$. In the
experiments, we focus on $k=4$ or $k=8$ auto-regressive groups.

\begin{figure}
\begin{center}
\includegraphics[width=.35\textwidth]{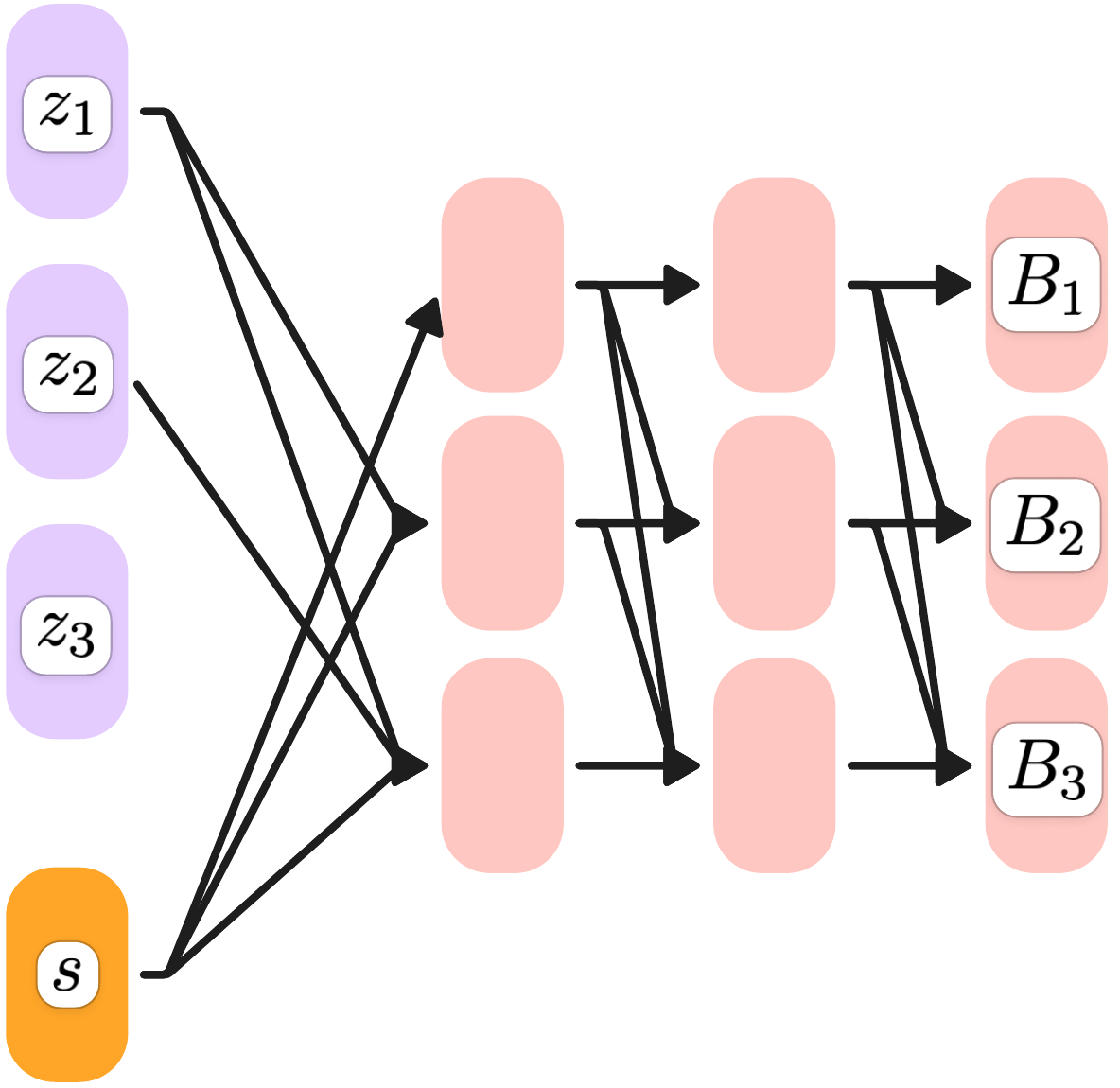}
\end{center}
\caption{An auto-regressive architecture for $B(s,z)$. The $i$-th block
of the output $B$ only depends on blocks $z_1,\ldots,z_{i-1}$ of the input
$z$.
In each layer, the weights from each block to the lower-ranking blocks
of the next layer have been removed. The state $s$ is still fed to every
block on the input layer.
\label{fig:Bschema}
}
\end{figure}

We employ a network architecture (Fig.~\ref{fig:Bschema}) in which each layer of $B_i$ has access
to the previous layers of all previous blocks $B_{1\ldots i}$: this
ensures good expressivity while preserving the auto-regressive property.
This allows for efficient evaluation: this is implemented as masks
on fully-connected layers for the full model $B$.

\bigskip

The following result extends the theorem from
\citet{fb_v1} for vanilla FB, to allow $B$ to depend on $z$.

\begin{theorem}
\label{thm:arfb}
Assume we have learned representations $F\from S\times A\times \R^d \to
\R^d$ and $B\from S\times \R^d\to \R^d$, as well as a parametric family
of policies $\pi_z$ depending on $z\in \R^d$, satisfying
\begin{equation}
\label{eq:FB}
\begin{cases}
 M^{\pi_z}(s_0,a_0,X) = \int_X \transp{F(s_0,a_0,z)}B(s,z) \,\rho(\d
 s), \qquad \forall s_0\in S, \, a_0\in A, X\subset S, z\in \R^d
\\
\pi_z(s) = \argmax_a \transp{F(s,a,z)}z,
\qquad\qquad\qquad\qquad \forall (s,a)\in S\times A, \, z\in \R^d.
\end{cases}
\end{equation}

Then the following holds. For any reward function $r$, if we can find a
value $z_r\in \R^d$ such that
\begin{equation}
\label{eq:z}
z_r=\E_{s\sim \rho}[r(s) B(s,z_r)]
\end{equation}
then $\pi_{z_r}$ is an optimal policy for reward $r$, and the optimal
$Q$-function is $Q^\star_r(s,a)=\transp{F(s,a,{z_r})}z_r$.

Moreover, if $B$ is auto-regressive, then the fixed point \eqref{eq:z}
always exists, and can be
computed directly, by iteratively computing each component
$z_i=\E[r(s)B_i(s,z_1,\ldots,z_{i-1})]$ for $i=1,\ldots,d$.  

\end{theorem}

Further theoretical properties and proofs are given in
Appendix~\ref{sec:theory}. In particular, Theorem~\ref{thm:univapprox}
establishes that autoregressive FB with two blocks is enough to represent
\emph{any} task encoding $r\mapsto z_r$, while vanilla FB is contrained
to a linear task encoding $r\mapsto z_r$. Thus, autoregressive FB is
inherently more expressive.

Training $F$ and $B$ for this setup is similar to \citet{fb_v1}, except
that $B$ depends on $z$, which has some consequences for minibatch
sampling, and results in higher variance.  The details are given in
Appendix~\ref{sec:arfbalgo} and Algorithm~\ref{alg:fb-aware}.  Training is based on the measure-valued
Bellman equation satisfied by $M^{\pi_z}$: we plug in the model
$M^{\pi_z}\approx \transp{F(z)}B(z)\rho$ in this equation and minimize
the Bellman gaps. 

There is little computational overhead compared to vanilla FB. In
practice, we enforce the auto-regressive property via a single neural
network $B$, by dropping a specific subset of the layer connections across
neurons (Fig.~\ref{fig:Bschema}). This implementation allows for
efficient training: given access to any specific $z$ and $s$, the output
$B(s,z)$ can be computed in a single forward pass. On the other hand, the
computation of the fixed point $z_r$ from \eqref{eq:z} requires several
forward passes through the network $B$ in Fig.~\ref{fig:Bschema}, but
this occurs only at test time when the reward function is known.

Auto-regressive FB models the successor features $M^{\pi_z}$ via a
model $M^{\pi_z}\approx \transp{F(z)}B(z)\rho$ with full
dependency on $z$, versus $M^{\pi_z}\approx \transp{F(z)}B\rho$ for
vanilla FB. This is more natural, especially for large $\gamma$. Indeed, 
for $\gamma\to 1$ we have
$M^{\pi_z}(s_0,a_0,\d s)=\frac{1}{1-\gamma} \mu_z(\d s)+o(1/(1-\gamma))$
with $\mu_z$ the stationary distribution of $\pi_z$, namely,
approximately rank-one with $z$ dependency on the $s$ part. The vanilla
FB model has no $z$ dependency on the $s$ part, only on the $(s_0,a_0)$
part, which means all stationary distributions $\mu_z$ must be
approximated using the shared features $B(s)$.
\NDY{could have some plots comparing the FB
training loss and
the FBAR training loss}

\subsection{Better Offline Optimization for FB: Advantage Weighting and Other
Improvements}
\label{sec:aw}

Like off-policy algorithms designed for the offline RL setting,
vanilla FB training appears prone to distribution shift, hindering
its ability to scale and to learn on datasets exhibiting mixtures
of behaviors for various tasks. Inspired by the recent analysis
\citep{asac}, we propose a set of modifications to FB training to
overcome these limitations.

\paragraph{Improved advantage weighting objective} We introduce an
alternative policy optimization step for FB, based on recent
analysis and advancements in offline RL algorithms with policy
constraints.
We use an improved version of the \emph{advantage-weighted} (AW) regression loss
\citep{awr}, commonly used in popular recent algorithms \citep{awac, IQL,
CRR, XQL, asac}. Following \citet{awac}, a first version starts with
sampling a
batch of $n$ transitions from the data, and updates the parametric policy
$\pi_\theta$ to optimize
\begin{gather}
\label{eq:sec3:awac_obj}
    \argmax_\theta\E_{(a_{1:n},s_{1:n})\sim \calB}\left[\sum^n_{i=1}w(s_i, a_i)\log \pi_\theta(a_i|s_i)\right],
    \\ \text{where} \quad w(s_i, a_i)=\frac{\exp(A_\phi(s_i,
    a_i)/\beta)}{\sum^n_{j=0}\exp(A_\phi(s_j, a_j)/\beta)}, \notag
    \\
    \text{and} \quad A_\phi(s,a)=Q_\phi(s, a) - \E_{a'\sim\pi}[Q_\phi(s,
    a')] \notag
\end{gather}
is the advantage function as estimated by the critic model $Q_\phi$. The
weights $w(s, a)$ are a weighted importance sampling (WIS) approximation
of $\exp(A_\phi(s, a)/\beta)/Z$, where $Z = \E_{s, a\sim \calB }\left[\exp(A_\phi(s,a)/\beta)\right]$.

In the FB framework, the policies are conditioned by the latent variable
$z$, and the $Q$-function estimate is $Q_\phi(s,a,z)=F_\phi(s,a,z)^Tz$.
Therefore, a direct transposition of \eqref{eq:sec3:awac_obj} to the FB
framework leads to the following objective for training the policy:
\begin{gather}
\label{eq:sec3:awac_fb}
    \argmax_\theta\E_{(a_{1:n},s_{1:n})\sim \calB, {z_{1:n}\sim Z}}\left[\sum^n_{i=1}w(s_i, a_i. z_i)\log \pi_\theta(a_i|s_i, z_i)\right],
    \\ \text{where} \quad w(s_i, a_i, z_i)=\frac{\exp(A_\phi(s_i, a_i, z_i)/\beta)}{\sum^n_{j=0}\exp(A_\phi(s_j, a_j, z_j)/\beta)}, 
\notag
    \\
	\quad A_\phi(s,a,z)=F_\phi(s, a, z)^Tz - \E_{a'\sim\pi_z}[F_\phi(s, a', z)^Tz], \notag
\end{gather}
The choice of sampling an independent $z_i$ for each sample $(s_i,a_i)$,
as opposed to using a single $z$ across a minibatch, was
made based on empirically faster learning.

However, the variance and bias of this weighted importance sampling approach have a linear
inverse relationship with the batch size $n$.
Instead, 
we propose to 
use modified weights $w'(s, a, z)$ that implement
\textit{improved weighted importance sampling} (IWIS), a simple change to
WIS
proposed by \citet{iwis}, shown to reduce the bias of WIS from $O(n^{-1})$
to $O(n^{-2})$.
Integrating IWIS yields our final \emph{policy improvement objective}:
\begin{gather}
\label{eq:sec3:iwis_fb}
    \argmax_\theta\E_{(a_{1:n},s_{1:n})\sim \calB, {z_{1:n}\sim Z}}\left[\sum^n_{i=1}w'(s_i, a_i. z_i)\log \pi_\theta(a_i|s_i, z_i)\right],
    \\ \text{where} \quad w'(s_i, a_i, z_i)\propto \frac{w(s_i, a_i,
    z_i)}{\sum_{j\neq i}w(s_j, a_j, z_j)} \quad \text{and} \quad
    \sum^n_{j= 0}w'(s_j, a_j, z_j)=1, \notag %
\end{gather}
and $w(s_i,a_i,z_i)$ is as in \eqref{eq:sec3:awac_fb}.

We validate the effect of IWIS in
Table~\ref{appB:tab:full_bench_res} (Appendix~\ref{sec:ablations}): FB-AW, which uses the IWIS weights
\eqref{eq:sec3:iwis_fb}, performs
slightly but consistently better than with the WIS weights
\eqref{eq:sec3:awac_fb}. Figure~\ref{sec4:fig:ex_perf_bias}
(Appendix~\ref{sec:ablations}) illustrates how AW negates the
overoptimism of vanilla FB for predicting future rewards.

\paragraph{Evaluation-based sampling.}
Furthermore, following~\citet{asac}, we integrate
\textit{evaluation-based sampling} (ES), an additional
component to mitigate the undesirable consequences of learning a Gaussian
policy, which is generally insufficient to capture the distribution from the
exponentiated advantages.
Namely, when deploying $\pi_\theta$ at test time, we approximate the argmax in
\eqref{eq:piz} by sampling $M$ actions
$a_1,\ldots,a_M$ from the trained policy $\pi(s)$, and perform the one with
the largest $Q$-value as predicted by $F_\phi(s, a_i, z)^Tz$.
The specific impact of this change is illustrated in
Fig.~\ref{fig:es_abl} (Appendix~\ref{sec:ablations}).

\paragraph{Uncertainty representation.}
To represent uncertainty in the model, we train two different networks $F_1$
and $F_2$ for the forward embedding, inspired by \citep{td3,
fb_0shot}. However, we introduce two changes.

In the Bellman equation, we
use the \emph{average} of the resulting two estimates of the
target successor measures, namely,
$\tfrac{1}{2}(F_1(s_{t+1}, a_{t+1}, z)^\top B(s') + F_2(s_{t+1}, a_{t+1},
z)^\top B(s'))$. This departs from \cite{fb_0shot}, which used the min between
the target successor measures, namely, $\min \{ F_1(s_{t+1},
a_{t+1}, z)^\top B(s'), \allowbreak F_2(s_{t+1}, a_{t+1}, z)^\top B(s') \}$,
in line with \citep{td3}.
Indeed, for $Q$-function estimates, a min might encode
some form of conservatism, but for successor measures, the interpretation
of a min is less direct. \footnote{For instance, since the $Q$-function
for reward $r$ is $Q=M.r$, taking the min of $M$ might encode a min for
a reward $r$ but a max for the reward $-r$.}

Finally, we use two fully parallel networks for $F_1$ and $F_2$, while 
\cite{fb_0shot} opted for a shared processing network with two
separate shallow heads for $F_1$ and $F_2$.

The specific impact of these changes is reported in
Table~\ref{appB:tab:full_bench_res} (Appendix~\ref{sec:ablations}).

\section{Evaluation}

\subsection{Algorithms and Baselines}

We mainly compare the following algorithms:
\begin{itemize}
\item ``Vanilla'' FB: the classical implementation of FB from \cite{fb_0shot} that employs TD3 policy improvement loss.

\item FB-AW (FB with advantage weighting): The FB method using the
advantage weighting components described in Section~\ref{sec:aw}.

\item FB-AWARE (FB with advantage weighting and auto-regressive
encoding): The FB method using both AW and the auto-regressive component
from Section~\ref{sec:ar}. For the auto-regressive part, we test either
$4$ or $8$ consecutive auto-regressive blocks for $B$.
\item On some environments (those where AW is not necessary to reach good
performance) we also include FB-ARE without the AW component.
\end{itemize}

Very few baselines provide true zero-shot behavioral foundation models
that can tackle any new task at test time with no fine-tuning. (More
baselines exist for the restricted case of goal-reaching tasks, namely, reaching
a given target state, see Section~\ref{sec:relatedwork}.)
In addition to the vanilla FB baseline, we include the following non-FB
baseline:
\begin{itemize}
\item Universal successor features \citep{borsa2018universal} based on
\emph{Laplacian eigenfunctions} as the base features \citep{fb_0shot}.
This version of successor features was found to perform best in
\cite{fb_0shot}. We denote it by LAP-AW, since we use the advantage
weighting as in FB-AW.
\end{itemize}

All the aforementioned variants of FB use the same architecture and
consistent hyperparameters.

\subsection{Datasets and Benchmarks}

We consider a series of environments, tasks, and datasets for these
environments, as follows.

\begin{itemize}

\item We start with the \emph{Jaco arm} domain~\citep{urlb}, a simple
robotic arm model. The tasks consist in reaching various target positions
(Section~\ref{sec:jaco}). This provides a test of spatial precision.

For this domain, we build a training dataset
via the \emph{RND} unsupervised exploration
method from~\citet{exorl}, which provides good data diversity if exploration is not too difficult in an
environment.

\item Next, we consider four standard domains from the \emph{DeepMind Control
(DMC)
Suite}~\citep{dmc}: Walker, Cheetah, Quadruped, and Humanoid. Since we
want to build behavior foundation models and not task-specific agents, on
top of the classical tasks for these environments (walk, run...) we
introduce a number of additional tasks such as bounce, flip, pullup...,
described in Appendix~\ref{sec:newtasks}.

For these domains, we consider two training datasets:
\begin{itemize}
\item We build a first dataset using RND, as for Jaco.
However, RND appears to provide insufficient exploration (particularly on
Humanoid); moreover, the RND trainset does not contain any purposeful
trajectories.

\item Therefore, we also train on the
\emph{MOOD} datasets from~\citet{asac}. MOOD
contains a mixture of behaviors, obtained as follows. For each
environment, a
small number of ``classical'' tasks are selected.
Then an online TD3 algorithm is used to train a classical agent for each
of these tasks. The set of trajectories produced by these agents during
training are then pooled and merged into a single dataset for the
environment. Thus,
the mixed-objective MOOD datasets include high-quality examples for a few
tasks in each environment.

Evaluation on the MOOD dataset must distinguish between tasks that
contributed to the dataset (\emph{in-dataset} tasks), and tasks that did
not (\emph{out-of-dataset} tasks). \footnote{We avoid ``in-distribution''
and ``out-of-distribution'', since FB is not trained on a distribution of
tasks but in an unsupervised way given the data.} A priori, one would expect the former
to be easier, as information from the original
single-task agents is present in the data.
To evaluate the ability of the FB models to generalize beyond
in-dataset tasks, we used the new tasks from Appendix~\ref{sec:newtasks} as
out-of-dataset tasks.
\end{itemize}

\item 
Finally, to test the generality of the approach, we also train FB,
FB-AW and FB-AWARE agents on the locomotion tasks of the \emph{D4RL benchmark}~\citep{d4rl}. Here we stick
to the original tasks in the benchmark, and compare FB performance to
the best task-specific offline RL agents in the literature. Since those
are single-task while FB is a generalist agent, this is a natural
\emph{topline}
for FB, so we expect FB to reach a good fraction of the performance of
the best task-specific agents, in line with the methodology of
\cite{fb_0shot}.

\end{itemize}

\subsection{Empirical Evaluation}
\label{sec:experiments}

We train FB, FB-AW, FB-AWARE (4 and 8 blocks) and LAP-AW on the four DMC
locomotion environments (Walker, Cheetah, Quadruped and Humanoid), as
well as the Jaco arm domain. We pretrain each model on both offline
datasets (MOOD and RND), and repeat each training 5 times (with different random seeds).

We evaluate each model on several downstream tasks per environment. For
each model and task, we sample 100,000 states $\{s\}$ from the offline
dataset and compute their corresponding task reward $\{ r(s)\}$ in order
to infer the task encoding vector $z_r$ (\eqref{eq:FBz} or
\eqref{eq:z}). Then we compute the cumulated reward
achieved by the policy $\pi_{z_r}$, computed using the task-specific
reward, and averaged over 100 episodes. Finally, we report the average and
variance of the cumulative reward over the 5 pre-trained models (with
different seeds).

\begin{figure}[t]
    \centering
    \includegraphics[width=0.95\linewidth]{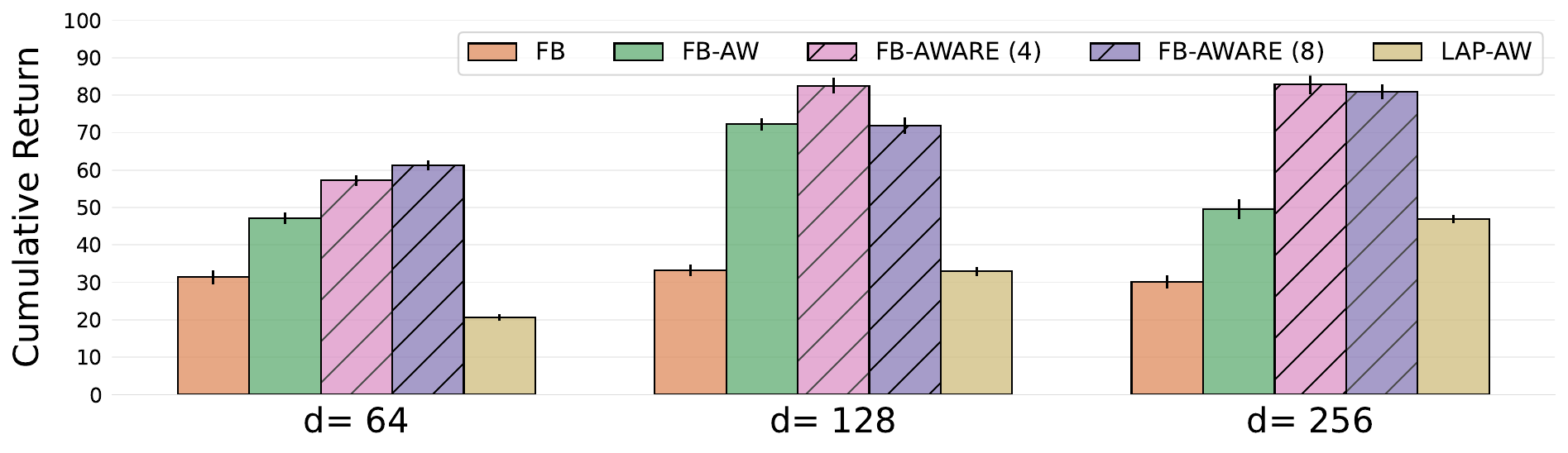}
    \caption{Average cumulative reward achieved by the algorithms,
    trained on RND dataset for different representation dimensions when
    aiming to reach goals (four randomly selected goals and four corner
    goals), in the Jaco arm environment.
    }
    \label{fig:random_jaco_rnd}
\end{figure}%

\subsubsection{Jaco Arm Results}
\label{sec:jaco}

We train each algorithm in the Jaco environment using the RND dataset,
for three choices of representation dimension: $d=64, 128, \text{ and } 256$. The tasks involve reaching a goal within an episode length of 250 time steps, with the agent receiving a reward of approximately 1 when the arm's gripper is close to the target goal specified by its $(x, y, z)$ coordinates.
For the goals, we included the four corners of the environment, plus
four goals selected at random (once and for all, common to all the
algorithms tested).

In Figure~\ref{fig:random_jaco_rnd}, we depict the average rewards
attained by each algorithm for reaching this mixture of goals.
The resulting goal-reaching rewards for dimensions 64, 128, and 256 are presented in Tables~\ref{tab:jaco_rnd_64}, \ref{tab:jaco_rnd_128}, and~\ref{tab:jaco_rnd_256}, respectively.

FB-AW significantly outperforms FB, more than doubling the score for the
best dimension $d=128$.  FB-AWARE with 4 autoregressive blocks further enhances performance by a
relative margin of about $15\%$.

\subsubsection{DMC Locomotion Results}
\label{sec:dmc_loc}

For Cheetah, Quadruped, Walker and Humanoid, the 
MOOD dataset results in substantially better models than the RND dataset,
whatever the algorithm (Appendix~\ref{sec:full_results},
Table~\ref{tab:dmc_mood} vs
Table~\ref{tab:dmc_rnd}).
This is especially striking for Humanoid, where RND does
not explore enough and even a classical single-task TD3 agent is hard to
train.
Therefore, we focus the discussion on MOOD, with full
RND results in Appendix~\ref{sec:full_results}.

\begin{figure}[t]
    \centering
    \includegraphics[width=0.95\linewidth]{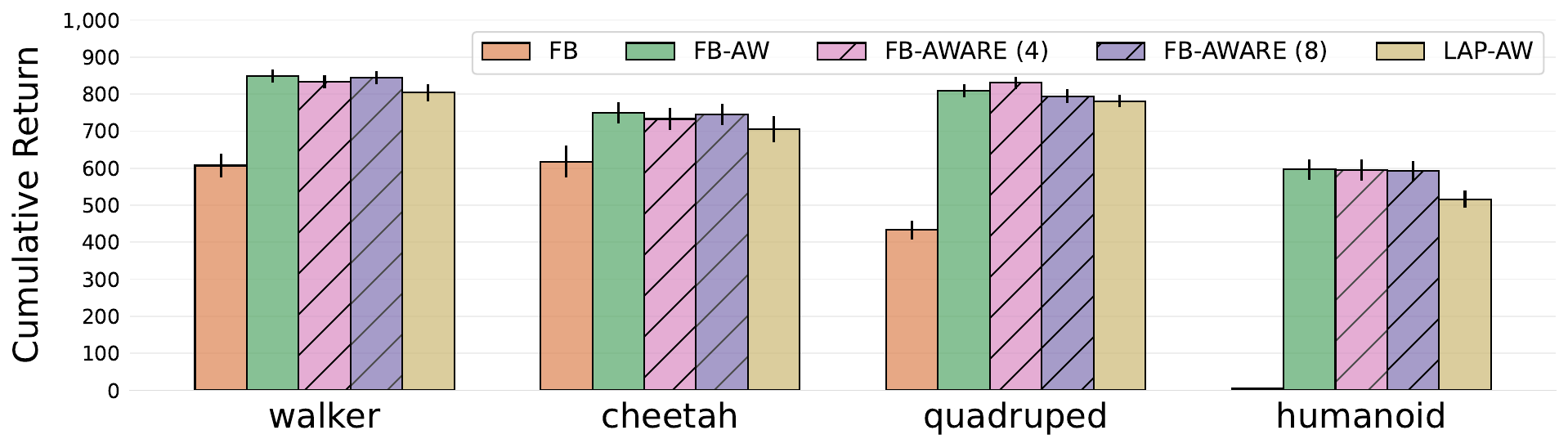}
    \caption{Averaged cumulative reward achieved by the algorithms on
    \emph{in-dataset tasks}, trained on MOOD
    dataset for DMC Locomotion.}
    \label{fig:train_tasks}
\end{figure}%

\begin{figure}[t]
    \centering
    \includegraphics[width=0.95\linewidth]{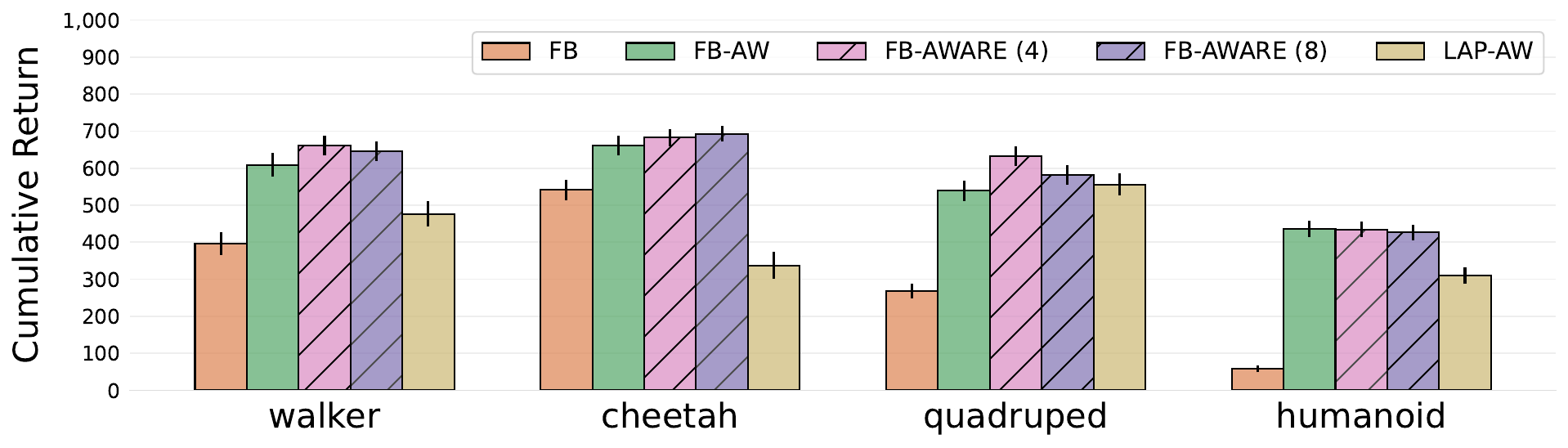}
    \caption{Average cumulative reward achieved by the algorithms on 
    \emph{out-of-dataset} tasks, trained on MOOD dataset for DMC
    Locomotion.}
    \label{fig:test_tasks}
\end{figure}%

Figures~\ref{fig:train_tasks} and~\ref{fig:test_tasks} report the results for the Cheetah,
Quadruped, Walker and Humanoid environments, using MOOD data for
training, evaluated on in-dataset tasks and out-of-dataset tasks
respectively. We used a fixed representation dimension for all algorithms
($d=64$ for Walker, Cheetah and Quadruped, $d=128$ for Humanoid). 

In this setup, the advantage weighting component proves crucial for
achieving satisfactory performance, particularly in the Humanoid
environment, where vanilla FB performance is near-zero. However, on the
lower-quality RND dataset, the advantage weighting component appears to
hurt performance (Appendix~\ref{sec:full_results},
Table~\ref{tab:dmc_rnd}). This is consistent with observations in
\citet{exorl} for the single-task setting, where conservative offline-RL
methods hurt performance on RND data.

LAP-AW does well at in-dataset tasks, but lags behind FB for
out-of-dataset tasks. This may be because LAP-AW's learned features are
closely tied to the in-dataset tasks, derived from the eigenfunctions of
the Laplacian of the behavior policy present in the dataset.

FB-AW and FB-AWARE exhibit the most favorable and consistent overall
performance. The autoregressive component provides a slight enhancement
in out-of-dataset tasks across all environments except from Humanoid.
This slight difference is not observed on the RND dataset
(Table~\ref{tab:dmc_rnd}).
The best-performing model overall is obtained with the MOOD dataset and FB-AWARE algorithm.

\subsubsection{Performance of FB-AW and FB-AWARE on D4RL }
\label{sec:d4rl}

Finally, to test the generality and robustness of these methods, we test
performance on the D4RL benchmark after reward-free training.  D4RL is a
ubiquitous benchmark, used by many recent offline RL research
for evaluation and comparison.

Here, we are comparing the multitask, unsupervised FB-AWARE agent to
task-specifics agents, so the performance of the latter are a natural
\emph{topline} for FB-AW and FB-AWARE. In line with \cite{fb_0shot}, we
expect FB-AWARE to reach a good fraction of the performance of the agents
trained specifically on each task.

So we compare FB-AWARE's performance with the results available for a
large pool of offline RL algorithms optimizing for an individual
objective with full access to rewards. This setting is quite different
from Section~\ref{sec:dmc_loc} and especially the MOOD datasets,
since D4RL datasets mostly comprise trajectories from agents trying to
accomplish a single task.

FB-AW and FB-AWARE's overall performance matches the task-specific recent
state-of-the-art from XQL and IQL. There is a slight advantage to
FB-AWARE over FB-AW, although this falls within the overall margin of
error.

This
establishes that zero-shot,
task-agnostic behavior foundation models trained via FB can compete 
with top task-specific agents for offline RL on standard benchmarks.

\begin{table}[t]

\caption{Performance on the popular locomotion-v2 and FrankaKitchen datasets from the D4RL
benchmark, comparing with recent offline RL algorithms (with
performance as reported in the literature).
FB-AWARE uses 8 AR groups. Other hyper-parameters are tuned
per-environment, consistently with the offline baselines.
}
\vspace{10pt}
\centering
\resizebox{1\columnwidth}{!}{
\begin{tabular}{@{}lccccccccc|ccc@{}}
\toprule
Dataset/Algorithm                       & BC    & 10\%BC & DT    & 1StepRL & AWAC  & TD3+BC & CQL   & IQL   & XQL   & FB& FB-AW               & FB-AWARE           \\ \midrule
                                        & \multicolumn{9}{c}{Reward-based learning}                                        & \multicolumn{3}{c}{Reward free learning} \\ \cmidrule(l){2-13} 
halfcheetah-medium-v2                   & 42.6  & 42.5   & 42.6  & 48.4    & 43.5  & 48.3   & 44.0  & 47.4  & 48.3  & 49.0±1.93 & 60.0±0.9            & 62.7±0.9           \\
hopper-medium-v2                        & 52.9  & 56.9   & 67.6  & 59.6    & 57.0  & 59.3   & 58.5  & 66.3  & 74.2  & 0.9±0.69 & 59.1±5.0            & 59.9±21.4          \\
walker2d-medium-v2                      & 75.3  & 75.0   & 74.0  & 81.8    & 72.4  & 83.7   & 72.5  & 78.3  & 84.2  & 0.5±0.9  & 80.5±11.7           & 89.6±0.8           \\
halfcheetah-medium-replay-v2            & 36.6  & 40.6   & 36.6  & 38.1    & 40.5  & 44.6   & 45.5  & 44.2  & 45.2  & 30.8±23.4  & 52.7±1.2            & 50.8±1.2           \\
hopper-medium-replay-v2                 & 18.1  & 75.9   & 82.7  & 97.5    & 37.2  & 60.9   & 95.0  & 94.7  & 100.7 & 16.4±2.9 & 87.1±3.7            & 89.6±4.0           \\
walker2d-medium-replay-v2               & 26.0  & 62.5   & 66.6  & 49.5    & 27.0  & 81.8   & 77.2  & 73.9  & 82.2  & 9.9±4.9  & 91.7±6.3            & 98.8±0.5           \\
halfcheetah-medium-expert-v2            & 55.2  & 92.9   & 86.8  & 93.4    & 42.8  & 90.7   & 91.6  & 86.7  & 94.2  & 91.7±6.7  & 99.6±0.8            & 100.1±0.7          \\
hopper-medium-expert-v2                 & 52.5  & 110.9  & 107.6 & 103.3   & 55.8  & 98.0   & 105.4 & 91.5  & 111.2 & 1.8±1.2 & 55.9±11.6           & 62.2±9.2           \\
walker2d-medium-expert-v2               & 107.5 & 109.0  & 108.1 & 113.0   & 74.5  & 110.1  & 108.8 & 109.6 & 112.7 & 0.3±0.8 & 109.6±1.3           & 105.8±0.6          \\
\hline \hline
\multicolumn{1}{c}{Locomotion-v2 total} & 466.7 & 666.2  & 672.6 & 684.6   & 450.7 & 677.4  & 698.5 & 692.6 & 752.9 & 22.3 & 696.2               & 719.5              \\
\hline 
kitchen-partial-v0 & 38.0 & -  & - & -  & - & -  & 49.8 & 46.3 & 73.7 & 4±4  &          47.0±4.5    & 52.5±9.4 \\
kitchen-mixed-v0 & 51.5 & -  & - & -  & - & -  & 51.0 & 51.0 & 62.5 & 5±5  &        48.5±7.2      & 53.5±3.8  \\
 \bottomrule
\end{tabular}

}
\end{table}

\subsection{Ablations and Discussion}

Appendix~\ref{sec:ablations} contains additional tests and ablations
concerning each of the components introduced, such as the impact of
dimension $d$, the number of auto-regressive blocks for FB-AWARE,
specific design choices for FB training ($B$ normalization, $z$
sampling), and the offline RL methods introduced in Section~\ref{sec:aw}
(advantage weighting, improved weighted importance sampling,
evaluation-based sampling, and uncertainty representation).

\paragraph{The impact of training data for learning behavior foundation
models.} Perhaps unsurprisingly, the dataset has a large impact on the
performance of behavior foundation models, as exemplified by the higher
scores of FB methods trained on MOOD vs RND for the Locomotion tasks. On
Humanoid, the combination of both MOOD and advantage weighting appears
necessary to reach any reasonable performance at all.

On the other hand, the best algorithm to train a behavior foundation
model also depends on the data available: with only RND data, advantage
weighting actually hurts performance (Table~\ref{tab:dmc_rnd}), although
with these data, performance is relatively poor anyway. This is
consistent with existing observations for classical single-task agents:
on RND data, TD3 is better than
conservative offline-RL approaches
\citep{exorl}.

Therefore, to train good behavior foundation models, it may be necessary
to train variants via several algorithms, and check the model's
performance on a few validation tasks.

\paragraph{Autoregressive FB and fine-grained behaviors.} The advantage
provided by autoregressive FB is clearer on the Jaco arm, and on
out-of-dataset tasks for DMC Locomotion. We can offer some speculative
explanations.

Part of the original motivation behind autoregressive FB was to provide
more precise task encoding. Indeed, vanilla FB projects tasks linearly
onto a subset of tasks, and this results in \emph{reward blurring} (loss
of detail in the actual task being optimized) \citep{fb_v1}. This
may correspond to selecting the largest eigenvalues of the successor
measure, since FB fixed points correspond to its eigenvalues
\citep{successorstates}. Intuitively, autoregressive features should
provide more detailed features based on a first set of coarse-grained
featues: for instance, if a target is roughly located in a region via the
top-level features $B_1$, then $B_2$ should learn more specific features
for the region given by $B_1$ (as encoded by $z_1$).

The Jaco arm domain aims at testing this intuition, by providing a task
in which a target must be precisely reached. The observed advantage of
autoregressive FB for this task is consistent with this intuition,
somewhat suggesting autoregressive features are effective at providing finer spatial precision.

One possible, though speculative, explanation for
the effect on out-of-dataset 
tasks is the following. In vanilla FB, the set of perfectly
optimized tasks is linearly spanned by $B$: thus, it is a linear space of
dimension $d$ among all possible tasks. In autoregressive FB, the subset of
rewards optimized by the features is nonlinear among all possible reward
functions, because the features $B_2$ depend on the continuous variable
$z_1$; for instance, with two auto-regressive blocks, one can represent
all
rewards $r(s)=z_1^T B_1(s)+z_2^T B_2(s,z_1)$. This results in a curved
manifold in the set of reward functions when varying $z_1$ and $z_2$. When facing a new arbitrary reward
function,
it may be easier to find a close match on this manifold than on a linear
subspace.

\paragraph{Limitations.} The environments considered
here are all noise-free, Markovian (history-free) continuous control environments. 

The effect of autoregressive FB is relatively modest in these
experiments. This is surprising given the huge theoretical change in
expressivity compared to vanilla FB. This suggests that the main limiting
factor in our suite of experiments may not be the expressivity of the
behavior foundation model, perhaps due to limited exploration in the training
datasets, or from the relative simplicity of the environments tested.

\section{Related Work}
\label{sec:relatedwork}

The forward-backward framework \citep{fb_v1} builds upon the principles
of successor features~\citep{barreto2017successor, zhang2017deep,
grimm2019disentangled}, the continuous extension to the canonical
successor representation \cite{dayan1993improving} and its continuous
state-space extension. However, in contrast to FB, this line of work has
mostly focused on constructing a set of features, linear w.r.t.\ downstream rewards, using apriori knowledge and heuristic measures such
as Laplacian eigenfunctions
\citep{borsa2018universal}. To this end \citet{fb_0shot} showed,
empirically, the superiority of end-to-end learning with FB as compared
to many such heuristics, in line with results in this same work. The
proposed autoregressive extension to FB could be also applied to this broader class of methods, a noteworthy direction for future work.

Goal-conditioned RL (GCRL) \citep{goal_cond_rl_review} is another area of research closely related to FB, which has seen notable successful applications in real-world robotics \citep{goal_cond_robotics_1_navig, goal_cond_robotics_2_vip, goal_cond_robotics_3_manip}. As with successor features, GCRL has traditionally relied on a priori knowledge taking the form of explicit demonstrations \citep{goal_cond_sup_imitation, goal_cond_sup_im_rl}, handcrafted subgoals \citep{her, goal_cond_subgoals, goal_cond_subgoals_im}, together with other coverage heuristics \citep{goal_cond_sup_iterated, goal_cond_from_bis} -- which have been used both to construct the space of goals and learn its relative multi-task policy. Moreover, in a similar fashion to FB, recent work also strived to model and tackle the GCRL problem with a principled contrastive-like end-to-end objective \citep{goal_cond_contr_0_C_learning, goal_cond_contr_1, goal_cond_contr_2_cdc}. However, GCRL is, by design, more restrictive than successor features and FB as it cannot capture tasks that go beyond reaching individual points in the space of goals.

Other works avoid the linearity constraint of successor features and
vanilla FB by explicitly relying on a prior over tasks. For instance, in
an approach akin to meta-RL but without hand-crafted tasks,
\citet{frans2024unsupervised} use a mixture of random MLPs, random linear
functions, and random goal-reaching to pre-train a set of policies
together with a encoder that quickly identifies a reward function from a
few reward samples. Contrary to FB, the dynamics of the environment plays
no role in building the set of features.

\NDY{should we cite more on
multitask RL/skill learning/meta-RL?}

\section{Conclusions}

Specific offline RL training techniques are necessary to build efficient
FB behavior foundation models in environments such as DMC Humanoid, and
can make the difference between near-zero and good performance.

Employing auto-regressive features greatly enhances the theoretical
expressivity of these foundation models, and improves spatial precision
and task generalization. The improvement is moderate in our setup,
perhaps indicating that BFM expressivity is not a key limiting factor for
these tasks.

These improvements bring zero-shot, reward-free FB BFMs on par with
single-task, reward-trained offline agents on a number of locomotion
environments.

\bibliography{main}
\bibliographystyle{icml2023}

\newpage
\appendix
\onecolumn

\section{Auto-Regressive FB: Extensions, Proofs, Algorithmic
Considerations}
\label{sec:theory}

\subsection{Proof and Extension of Theorem~\ref{thm:arfb}}

Here we prove Theorem~\ref{thm:arfb} and extend it in two directions similar
to \cite{fb_v1}. 

The first extension concerns \emph{goal spaces}, and is useful when we
know in advance that the rewards functions of interest will not depend on
the whole state. For instance, in a multi-agent setting, the reward of an
agent may depend only on its own state, but it must still observe the
whole state to make decisions. This is formalized by assuming that the
reward function only depends on some variable $g=\phi(s)$ rather than the
whole state $s$. Then we can learn with $B(g)$ instead of $B(s)$ (while
$F$ and policies still require the full state).

The second extension only uses the finite-rank $\transp{F}B$ model for the
\emph{advantage} functions, namely, the model is $M^{\pi_z}(s_0,a_0,\d
s)\approx \transp{F(s_0,a_0,z)}B(s',z)+\bar m(s_0,z,s)$ where $\bar m$ is
any function independent of the actions. This lifts part of the
finite-rank restriction, since $\bar m$ itself is unconstrained: the
finite-rank FB model is only applied to the differential effect of actions
on top of the baseline model $\bar m$.

\begin{definition}[Extended forward-backward representation of an MDP]
\label{def:fb2}
Consider an MDP with state space $S$ and action space
$A$.
Let $\phi\from S\times A\to G$ be a function from state-actions to some
goal space $G=\R^k$.

Let $Z=\R^d$ be some representation space. Let
\begin{equation}
F\from S\times A\times Z\to Z,\qquad B\from G\times Z\to Z,\qquad
\bar m\from S\times Z\times G\to \R
\end{equation}
be three functions. For each $z\in Z$, define the policy
\begin{equation}
\pi_z (a|s) \deq \argmax_a \transp{F(s,a,z)}z.
\end{equation}

Let $\rho$ be any measure over the goal space $G$.

We say that $F$, $B$, and $\bar m$ are an \emph{extended forward-backward
representation} of the MDP with respect to $\rho$, if the following
holds:
for any $z\in Z$, any state-action $(s,a)$
the successor measure $M^{\pi_z}$ of policy $\pi_z$
is given by
\begin{equation}
\label{eq:extendedFB}
M^{\pi_z}(s,a,X)=\int_{g\in X}\left(\transp{F(s,a,z)}B(g,z)+\bar m(s,z,g)\right)\rho(\d g)
\end{equation}
for any goal subset $X\subset G$.
\end{definition}

\begin{theorem}[Forward-backward representation of an MDP, with features as
goals]
\label{thm:arfb2}
Consider an MDP with state space $S$ and action space
$A$.
Let $\phi\from S\times A\to G$ be a function from state-actions to some
goal space $G=\R^k$.

Let $F$, $B$, and $\bar m$ be an extended forward-backward representation
of the MDP with respect to some measure $\rho$ over $G$.

Then the following holds.
Let $R\from S\times A \to \R$ be any bounded reward function, and assume that
this reward function depends only on $g=\phi(s,a)$, namely, that there exists a
function $r\from G\to \R$ such that $R(s,a)=r(\phi(s,a))$.

Assume that there exists $z_R\in \R^d$ satisfying
\begin{equation}
\label{eq:zR2}
z_R= \int_{g\in G} r(g)B(g,z_R) \,\rho(\d g).
\end{equation}
This is always the case if $B$ is auto-regressive.

Then:
\begin{enumerate}
\item $\pi_{z_R}$ is an optimal policy for reward $R$ in the MDP.

\item 
The optimal $Q$-function $Q^\star_R$ for reward $R$ is
\begin{equation}
Q^\star_R(s,a)=\transp{F(s,a,z_R)}z_R+
\int_{g\in G} \bar m(s,z_R,g)r(g)\,\rho(\d g).
\end{equation}
The last term does not depend on the action $a$, so computing the $\bar
m$ term is not necessary to obtain 
the advantages $Q^\ast(s,a)-Q^\ast(s,a')$ or the optimal policies.
\end{enumerate}
\end{theorem}

Theorem~\ref{thm:arfb2} implies Theorem~\ref{thm:arfb}, by taking
$\phi=\Id$ and $\bar m=0$. 

There is an important difference between this result and the
corresponding statement in non-auto-regressive FB (Theorem 4 in
\cite{fb_v1}). For non-autoregressive FB, the FB model provides the
$Q$-functions of all policies $\pi_z$ for all rewards $R$, even if $z\neq
z_R$. Namely, $Q^{\pi_z}_R(s,a)=\transp{F(s,a,z)}z_R$ for all pairs
$(z,R)$ (in the case $\bar m=0$). Here, this only holds when $z=z_R$.
Classical, non-auto-regressive FB provides more $Q$-functions than strictly needed
to obtain the policies: it models the $Q$-functions of $\pi_z$ even for
rewards unrelated to $z$. With auto-regressive FB, the additional
expressivity of the model comes at the price of getting less information
about the other $Q$-functions.

\begin{proof}[Proof of Theorem~\ref{thm:arfb2}]
Let $M^\pi$ be the successor measure 
of policy $\pi$. Let
$m^{\pi}$ be the density of $M^{\pi}$ with respect to $\rho$.
Let $R(s,a)=r(\phi(s,a))$ be a reward function as in the statement of the theorem.

By Proposition~16 in \cite{fb_v1}, 
The $Q$-function of $\pi$ for the reward $R$ is
\begin{align}
Q^\pi_R(s,a)
=\int_g r(g) M^\pi(s,a,\d g)
\end{align}

Let $z_R$ satisfy the fixed point property \eqref{eq:zR2}, and let us
take
$\pi=\pi_{z_R}$.
By definition of an extended FB representation, we have
\begin{align}
Q^{\pi_{z_R}}_R(s,a)
&=\int_g r(g) M^{\pi_{z_R}}(s,a,\d g)
\\&=
\int_g r(g) (\transp{F(s,a,z_R)}B(g,z_R)+\bar
m(s,z_R,g))\rho(\d g).
\\&=
\transp{F(s,a,z_R)}\int_g r(g) B(g,z_R)\rho(\d g)+\int_g r(g) \bar
m(s,z_R,g))\rho(\d g).
\end{align}
But thanks to the fixed point property \eqref{eq:zR2}, we have
$\int_g r(g)B(g,z_R)\rho(\d g)=z_R$.
Therefore, the $Q$-function of $\pi_{z_R}$ for reward $R$ is
\begin{align}
Q^{\pi_{z_R}}_R(s,a)=\transp{F(s,a,z_R)}z_R+\int_g r(g) \bar
m(s,z_R,g))\rho(\d g).
\end{align}

We have to prove that this is the optimal $Q$-function for $R$. A pair of
a $Q$-function and policy $\pi$ are optimal for $R$ if and only if
simultaneously
$\pi(a|s)=\argmax_a Q(s,a)$ and $Q=Q^{\pi}_R$. Here, by definition of the
policies $\pi_z$, we have
\begin{align}
\pi_{z_R}(a|s)&=\argmax_a \transp{F(s,a,z_R)}z_R
\\&=\argmax_a Q^{\pi_{z_R}}_R(s,a)
\end{align}
since the additional term $\int_g r(g) \bar
m(s,z_R,g))\rho(\d g)$ in $Q^{\pi_{z_R}}_R$ does not depend on $a$.

Therefore, $Q^{\pi_{z_r}}_R$ and $\pi_{z_R}$ are optimal for reward $R$,
which ends the proof.
\end{proof}

\begin{theorem}[Auto-regressive features with two levels are a universal approximator for
task encoding]
\label{thm:univapprox}
Assume the state space is finite, so that a reward function
is an element of $\R^{\#S}$. Then, for any continuous task
encoding function $\zeta\from  \R^{\#S} \to \R^d$ mapping rewards $r$ to
task encodings $z=\zeta(r)$, such that $z_r=0$ for $r=0$, there exist neural
networks $B_1(s)$ and $B_2(s,z)$ approximating $\zeta$, namely, for any
$r$,
\begin{equation}
\zeta(r)\approx \E_{s\sim \rho} [r(s)B_2(s,z_1)], \qquad z_1=\E_{s\sim
\rho}[r(s)B_1(s)]
\end{equation}
up to an arbitrary precision.
\end{theorem}

\begin{lemma}
Let $f\from \R^n\to \R^d$ be a $C^3$ function such that $f(0)=0$. Then
there exists a continuous matrix-valued function $g\from \R^n\to \R^{n\times d}$
such that for any $x\in \R^n$,
\begin{equation}
f(x)=g(x)\cdot x.
\end{equation}
\end{lemma}

\begin{proof}[Proof of the lemma]
By working on each output component of $f$ separately, we can assume that
$d=1$. Thus, we have to prove that for any $C^3$ function $f\from \R^n\to
\R$ with $f(0)=0$, there exists a vector-valued function $g\from \R^n\to
\R$ such that $f(x)=\transp{g(x)}x$.

Since $f$ is $C^3$ with $f(0)=0$, we can write its Taylor expansion
\begin{equation}
f(x)=\transp{D}x+\frac12 \transp{x}Hx+R(x)
\end{equation}
where $D=\partial_x f(0)$ is the gradient of $f$ at $x=0$,
$H=\partial^2_x f(0)$ its Hessian, and where the remainder $R$ is
$O(\norm{x}^3)$.

A priori, the function $R(x)/\norm{x}^2$ is defined everywhere except
$x=0$. But since $R$ is $O(\norm{x}^3)$, $R(x)/\norm{x}^2$ tends to $0$
for $x=0$, so it is a well-defined continuous function on the whole
domain.

Thus, let us set
\begin{equation}
g(x)\deq D+\frac12 Hx +\frac{R(x)}{\norm{x}^2}\, x.
\end{equation}
Then, by construction, $\transp{g(x)}x=f(x)$ as needed.
\end{proof}

\begin{proof}{Proof of Theorem~\ref{thm:univapprox}}
By the lemma, there exists a matrix-valued function $g$ such that $\zeta(r)=g(r)\cdot
r$ for $r\in \R^{\#S}$.

Define $B_1(s)$ to be the one-hot encoding $B_1(s)=\1_s/\rho(s)\in \R^{\#S}$,
where $\1_i$ denotes the vector with all zeroes except a $1$ at position
$i$. By universal approximation theorems for neural networks, this choice
of $B_1$ can be realized by a neural network with arbitrary good
approximation (actually a one-layer neural network with identity weights
and no activation function).

Then
\begin{equation}
z_1=\E_{s\sim \rho} [r(s)B_1(s)]=\sum_s \rho(s) r(s) \1_s/\rho(s)=r.
\end{equation}
Then in turn, for any $B_2$, 
\begin{equation}
z_2=\E_{s\sim \rho} [r(s)B_2(s,z_1)]=\sum_s \rho(s) r(s)
B_2(s,r).
\end{equation}

For each $s$ and $r$, $B_2(s,r)$ is an element of $\R^d$. For each
component $1\leq i\leq d$, define
\begin{equation}
B_2(s,r)_i\deq g_{is}(r)/\rho(s)
\end{equation}
where $g(r)$ is the matrix defined above. Then we have, by construction
\begin{equation}
\sum_s \rho(s) r(s)
B_2(s,r)_i=\sum_s r(s) g_{is}(r)
\end{equation}
namely
\begin{equation}
B_2(s,r)=g(r)\cdot r=\zeta(r)
\end{equation}
as needed.

By universal approximation theorems for neural networks, it is possible to
realize this choice of $B_2$ by a neural network, with arbitrarily good
approximation.

Note: the principle of this proof extends to continuous state spaces, by
taking a partition of unity for $B_1$ instead of a one-hot encoding,
though this results in further approximation errors.
\end{proof}

\subsection{Training Loss and Algorithmic Considerations for
Auto-Regressive FB}
\label{sec:arfbalgo}

\paragraph{Loss for FB-AR; sampling.} While plain FB represents the
successor measures as $M^{\pi_z}(s,a,\d
s')\approx\transp{F(s,a,z)}B(s')\rho(\d s')$, auto-regressive FB uses $M^{\pi_z}(s,a,\d
s')\approx\transp{F(s,a,z)}B(s',z)\rho(\d s')$. The training principle
remains the same: learn $F$ and $B$ to minimize the error in this
approximation. However, this leads to several changes in practice.

The error of the FB model can be measured as in plain FB, based on the
Bellman equation satisfies by $M^{\pi_z}$ (see
Appendix~B in \cite{fb_0shot}). For each value of $z$, the
Bellman loss on $M^{\pi_z}(s,a,\d
s')-\transp{F(s,a,z)}B(s')\rho(\d s')$ is
\begin{align}
\label{eq:FBloss}
\mathcal{L}(F, B)
 & =
 \E_{\substack{(s_t, a_t,s_{t+1}) \sim \rho \\ s' \sim \rho}} \left[
 \Big( F(s_t, a_t,z)^\top B(s',z) - \gamma \bar{F}(s_{t+1},
 \pi_z(s_{t+1}),z)^\top \bar{B}(s',z)  \Big)^2 \right]
\nonumber \\
 &\phantom{=}- 2 \E_{(s_t, a_t,s_{t+1}) \sim \rho } \left[ F(s_t,
 a_t,z)^\top B(s_{t+1},z) \right]
 + \texttt{Const}
\end{align}
where $\texttt{Const}$
is a constant term that we can discard since it does not depend on $F$
and $B$. The only difference with plain FB is that now $B$ depends
on $z$.

For training, the variable $z$ is sampled as in \cite{fb_0shot}: namely,
$z$ is sampled 50\% of the time from a standard $d$-dimensional Gaussian
(and later normalized) and 50\% of the time from the $B$ representation
of a state ramdomly sampled from the training data (computed every step
for normal FB and computed every 32 steps for AR FB for speed
\NDY{unclear}). Fig.~\ref{fig:z_B_ablation}
(Appendix~\ref{sec:ablations}) tests the effect of only
sampling $z$ from a Gaussian.

For reasons discussed below, for FB we follow the original strategy of
sampling a different $z$ for each sampled transition, but for FB-AR we
sample a unique $z$ for the whole batch.

\NDY{the above assumes deterministic $\pi_z$}

Algorithm~\ref{alg:fb-aware} implements this loss, together with sampling of $z$.

\paragraph{Minibatch handling, and increased variance for auto-regressive
FB.} Having $B$ depend on $z$ has practical
consequences for variance and minibatch sampling. In vanilla FB,
starting from \eqref{eq:FBloss}, it is possible to sample a minibatch of
$N$ transitions $(s_t,a_t,s_{t+1})$, a minibatch of $N$ values of $s'$, choose a
different value of $z$ for each $s_t$ in the minibatch, compute the $N$ values of
$F(s_t,a_t,z)$ and $B(s')$, and compute the $N^2$ dot products
$\transp{F(s_t,a_t,z)}B(s')$ involved in the loss, at a cost of $N$
forward passes through $F$ and $B$ (Appendix A in \cite{fb_v1}).

In auto-regressive FB, the value of $z$ must be the same for $F$ and $B$
in the loss \eqref{eq:FBloss}. Therefore, contrary to plain FB, we only
sample a \emph{single} value of $z$ for the minibatch ($I=1$ in
Algorithm~\ref{alg:fb-aware}). Then we can
compute the $N$ values $F(s_t,a_t,z)$ and $B(s',z)$, and the $N^2$ dot products
$\transp{F(s_t,a_t,z)}B(s',z)$ for the loss \eqref{eq:FBloss}. This
appears in Algorithm~\ref{alg:fb-aware}.

Using a single value of $z$ per minibatch results in increased variance
of auto-regressive FB with respect to vanilla FB and longer training
times, which we observe in practice.

Indeed, if we want to keep using different values of $z$ for each
$F(s_t,a_t,z)$, we must either compute many more values $B(s',z)$ (one
for each pair $s'$ and $z$), or use fewer dot products by computing fewer
values of $B(s',z)$ and only using them with $F(s_t,a_t,z)$ with the same
$z$.

More precisely, in general we can proceed as follows: Let $k$ be a
hyperparameter controlling the number of distinct values of $z$ we will
use in the minibatch. We sample $N$ transitions $(s_t,a_t,s_{t+1})$, $N$
values of $s'$, and split these samples into $k$ groups of size $N/k$.
For each group, we sample a value of $z$, we compute the values
$F(s_t,a_t,z)$ and $B(s',z)$ for the states in that group, and we
use all dot products $\transp{F(s_t,a_t,z)}B(s',z)$ within that group.

Thus, the solution we used just has $k=1$: the same value of $z$ is used
throughout the minibatch, so we can compute $N^2$ dot products
$\transp{F(s_t,a_t,z)}B(s',z)$ with only $N$ forward passes through $F$
and $B$. The other extreme would be $k=N$: for each $(s_t,a_t,s_{t+1})$,
we pick a value of $z$ and a value of $s'$, compute $F(s_t,a_t,z)$ and
$B(s',z)$, and form the dot product $\transp{F(s_t,a_t,z)}B(s',z)$. This
uses more distinct values of $z$, but has only $N$ dot products. In
practice this would mean a reduced variance from sampling $z$, but an
increased variance from sampling $s'$. Our chosen option $(k=1)$ has the
opposite trade-off. This choice was based on preliminary results showing that, while
using a higher $k$ appeared to learn slightly faster at the beginning,
the fewer dot products led to convergence to slightly lower
performance.

To some extent, this effect may represent a hindrance for autoregressive
FB compared with vanilla FB. However, Fig.~\ref{fig:zperbatch} (Appendix~\ref{sec:ablations}) shows that
this effect is limited: indeed, vanilla FB with a single $z$ per
minibatch performs only slightly worse than vanilla FB with many $z$'s.

\paragraph{Normalization of $B$.}
As in \cite{fb_0shot}, to improve numerical conditioning on $B$, we use
an auxiliary orthonormalization loss which ensures that $B$ is
approximately an orthognoal matrix. Indeed, one can change $F$ and $B$
without changing the $B$ model, by $F\gets FC$ and $B \gets
B(\transp{C})^{-1}$ for any invertible matrix $B$, because the FB model
only depends on the values of $\transp{F}B$ \citep{fb_v1}. In the case of
auto-regressive features, since $B$ depends on $z$, we can do this
normalization separately for each $z$, without impacting the FB model.
Explicitly, the orthonormalization loss is
\begin{align}
\mathcal{L}_{\mathrm{norm}}(B)
&\deq
\E_z
\norm{\E_{s\sim \rho} [B(s,z)\transp{B(s,z)}]-\Id}^2_{\mathrm{Frobenius}}
\\&=
\E_z\E_{s\sim \rho,\, s'\sim \rho}
\left[ (\transp{B(s,z)}
B(s',z))^2
- \norm{B(s,z)}^2_2
- \norm{B(s',z)}^2_2\right]+\texttt{Const}.
\end{align}
where $z$ is sampled as for the main
loss.

Moreover, we further normalize the output of $B$: for each input $(s,z)$
and each auto-regressive block in the output of $B$, we set $B(s,z)\gets
B(s,z)\frac{\sqrt{d_k}}{\norm{B(s,z)}}$ with $d_k$ the size of the block
$k$, so that each output component is of size approximately $1$. 
Fig.~\ref{fig:z_B_ablation}
(Appendix~\ref{sec:ablations}) tests the effect of not using this
output normalization.

\paragraph{Scale invariance in the reward, and normalization of $z$.} In
reinforcement learning, the optimal behavior is the same for reward $r$
or reward $\alpha.r $ with any $\alpha>0$. This inherent invariance property can be directly enforced within
the architecture of FB to help with learning.

In non-autoregressive FB models, since $z=\E_s[r\times B(s)]$, the scaling $r\gets \alpha\times r$ directly translates to a scaling in the task representation $z\gets \alpha z$. Thus, one simple way to enforce invariance with respect to the rewards' scale is to always normalize $z$ to a fixed norm (in practice, norm
$\sqrt{d}$, so each of $z$ value value expectedly has a magnitude close to $1$). We denote this operation by the preprocessing function $f{z}=\bar{z}=\sqrt{d}\times \frac{z}{||z||}$, which we apply when feeding any $z$ to $\pi_\theta$ and $F_{\phi}$. Hence, this ensures that 
$F(s,a,z)=F(s,a,\alpha\times z)$ and $\pi_{z}=\pi_{\alpha\times z}$ by construction.
As a result, the predictions and behaviors made when facing rewards $r$ and $\alpha \times r$
are exactly the same.

However, in our auto-regressive FB models where $z=\E_s[r\times B(s, z)]$, the same strategy cannot be applied when feeding $z$ as input to $B$. This is because we never have access to the full $z$ until the very end of the inference procedure. Hence, given a reward $r$, we still want to have scale-invariance in $B$ but have no means of performing a standard input normalization, as we have no access to the magnitude of the resulting $z$.

A first strategy to counteract this limitation could be to normalize $z$ only within each auto-regressive group, $f_g(z)=\{f_g(z)_1, \dots, f_g(z)_n\}=\{\bar{z}_1, \dots, \bar{z}_n\}$. While this would enforce a fixed scale and not pose test-time issues it would non-trivially affect the information available to $B$ about the actual task $z$, losing any notion of relative magnitude between different groups. For instance, in case for $n=d$, each normalized component $\bar{z}_k$ would be a binary scalar %
only preserving the \textit{sign} from the corresponding task representation $z_k$. 

We overcome these limitations by designing a new `\textit{residual autoregressive} normalization strategy, $f_{ar}$ compatible with the requirements of $B$'s inference while still fully and exclusively preserving the information about $z$'s direction as with traditional normalization. 

Our strategy achieves these properties by using an iterative normalization scheme for each autoregressive component in its output $f_{r}(z)=\{f_r(z)_1, \dots, f_r(z)_n\}$. As with the aforementioned naive approach, we start with normalizing the first component $z_1$ within itself: $f_r(z)_1=\bar{z}_1$. Then, we proceed to \textit{residually normalize} all other $z_k$, also making use of all previous autoregressive components $z_1, \dots z_{k-1}$:
\begin{equation*}
f_r(z)_k=\frac{z_{k}}{||z_{1:k}||},
\end{equation*} %
where $z_{1:k}$ simply corresponds to the concatenated first $k$
auto-regressive components of $z$. Finally, we also rescale each $k^{th}$
component by a constant factor $\sqrt{\sum^k_{i=0} d_i}$ to avoid biasing
later components to have a smaller magnitude at initialization and
incentivize $||f_{r}(z)||\approx \sqrt{d}$. We note there is a bijective
map between traditional normalization and this auto-regressive scheme, thus, preserving the full information of the direction component of $z$ without requiring the full vector.

\edo{perhaps, we can have a small `proposition', the 1-to-1 mapping or residual AR norm with the direction component should be quite easy to prove via contradiction.}

\subsection{Algorithms}
\label{app: algos}

FB-AWARE training is described in Algorithm~\ref{alg:fb-aware}. For
reference,
Algorithm~\ref{alg:fb} describes vanilla FB training.

\begin{algorithm}[h!]
	\caption{FB training}\label{alg:fb}
\begin{small}
	\begin{algorithmic}[1]
 \State \textbf{Inputs} Offline dataset $\mathcal{D}$, number of ensemble
 networks $M$ for $F$, randomly initialized network $\{F_{\theta_m}\}_{m \in [M]}$, $B_\omega$ and $\pi_\phi$, transition mini-batch size $I$ mixing probability $\tau_{\mathrm{mix}}, $Polyak coefficient $\zeta$, orthonormality regularisation coefficient $\lambda$.
    \vspace{0.2cm}
	\For{$t = 1, \dots$}
 \State { \bf \textcolor{gray!50!blue}{/* \texttt{Sampling}}}
        \State Sample $I$ latent vectors \\
        \State $z \sim \left\{
    \begin{array}{lll}
        \mathcal{N}(0, I_d) &  & \mbox{with prob } 1 - \tau_{\mathrm{mix}} \\
        B(s) & \mbox{where } s \sim \mathcal{D},& \mbox{with prob } \tau_{\mathrm{mix}} \\
    \end{array}
\right.$
    \State $z \leftarrow \sqrt{d} \frac{z}{\| z\|}$
    \State Sample a mini-batch of $I$ transitions $\{(s_{i}, a_{i}, s_{i}')\}_{i\in [I]}$ from $\mathcal{D}$
    
    \State {\bf  \textcolor{gray!50!blue}{/* \texttt{Compute FB loss}}}
	 \State Sample $a_{i}' \sim \pi_\phi(s_{i}',z_i)$ for all $i\in[I]$
	\State $\mathscr{L}_{\texttt{FB}} (\theta_m, \omega) =\frac{1}{2 I (I-1)} \sum_{j \neq k} \left( F_{\theta_k}(s_{i}, a_{i},z_i)^\top B_\omega(s_{k}')  - \gamma \frac{1}{M}\sum_{m \in [M]} F_{\theta_m^{-}}(s_{i}',a_{i}',z_i)^\top B_{\omega^{-}}(s_{k}')  \right)^2$\\ 
	 $\qquad\qquad\qquad - \frac{1}{I} \sum_{i} F_{\theta_k}(s_{i}, a_{i},z_i)^\top  B_\omega(s_{i}'), \quad  \forall m \in [M]$
	
	\State {\bf  \textcolor{gray!50!blue}{/* \texttt{Compute orthonormality regularization loss}}}
	\State $\mathscr{L}_{\texttt{ortho}} (\omega) = \frac{1}{2 I(I-1)}  \sum_{i \neq k} (B_{\omega}(s_{i}')^\top B_{\omega}(s_{k}')^2 - \frac{1}{I} \sum_{i} B_{\omega}(s_{i}')^\top B_{\omega}(s_{i}') $
	\State {\bf  \textcolor{gray!50!blue}{/* \texttt{Compute actor loss}}}
	\State Sample ${a}^\phi_{i} \sim \pi_\phi(s_{i},z_i)$ for all $i\in[I]$
	\State $\mathscr{L}_{\texttt{actor}} (\phi) = -\frac{1}{I}\sum_{i} \left( \min_{m\in[M]} F_{\theta_m}(s_{i},{a}^\phi_{i} ,z_i)^T z_i \right)$
	\State {\bf  \textcolor{gray!50!blue}{/* \texttt{Update all networks}}}
	\State $\theta_m \leftarrow \theta_m - \xi \nabla_{\theta_m}  (\mathscr{L}_{\texttt{FB}}(\theta_k, \omega)$  for all $m\in[M]$
	\State $\omega \leftarrow \omega - \xi \nabla_{\omega}  (\sum_{l\in[m]} \mathscr{L}_{\texttt{FB}}(\theta_l, \omega) + \lambda \cdot  \mathscr{L}_{\texttt{ortho}}(\omega))$
	\State $\phi \leftarrow \phi - \xi \nabla_{\phi} \mathscr{L}_{\texttt{actor}} (\phi)$
	\State {\bf  \textcolor{gray!50!blue}{/* \texttt{Update target networks}}}
	\State  $\theta_m^{-} \leftarrow \zeta \theta_m^{-} + (1-\zeta) \theta_m$  for all $m\in[M]$ 
	\State  $\omega^{-} \leftarrow \zeta \omega^{-} + (1-\zeta) \omega$
\EndFor
 \end{algorithmic}
 \end{small}
	\end{algorithm}

\begin{algorithm}[h!]
	\caption{FB-AWARE training}\label{alg:fb-aware}
\begin{small}
	\begin{algorithmic}[1]
 \State \textbf{Inputs} Offline dataset $\mathcal{D}$, number of ensemble networks $M$, randomly initialized network $\{F_{\theta_m}\}_{m \in [M]}$, $B_\omega$ and $\pi_\phi$, transition mini-batch size $J$, latent vector mini-batch size $I$, number of autoregressive blocks $K$, mixing probability $\tau_{\mathrm{mix}}, $Polyak coefficient $\zeta$, orthonormality regularisation coefficient $\lambda$, temperature $\beta$.
    \vspace{0.2cm}
	\For{$t = 1, \dots$}
 \State { \bf \textcolor{gray!50!blue}{/* \texttt{Sampling}}}
        \State Sample $I$ latent vectors \\
        \State $z \sim \left\{
    \begin{array}{lll}
        \mathcal{N}(0, I_d) &  & \mbox{with prob } 1 - \tau_{\mathrm{mix}} \\
        (z_1, \ldots, z_K) = (B_1(s), \ldots, B_K(s, z_1, \ldots, z_{K-1})) & \mbox{where } s \sim \mathcal{D},& \mbox{with prob } \tau_{\mathrm{mix}} \\
    \end{array}
\right.$
    \State $z \leftarrow \sqrt{d} \frac{z}{\| z\|}$
    \State Sample a mini-batch of $I \times J$ transitions $\{(s_{i, j}, a_{i, j}, s_{i, j}')\}_{i\in [I], j \in [J]}$ from $\mathcal{D}$
    
	 \State Sample $a_{i, j}' \sim \pi_\phi(s_{i, j}',z_i)$ for all $i\in[I], j \in [J]$
	\State $\mathscr{L}_{\texttt{FB}} (\theta_m, \omega) =\frac{1}{2 I J (J-1)} \sum_{i} \sum_{j \neq k} \left( F_{\theta_k}(s_{i, j}, a_{i, j},z_i)^\top B_\omega(s_{i, k}', z_i)  - \gamma \frac{1}{M}\sum_{m \in [M]} F_{\theta_m^{-}}(s_{i, j}',a_{i, j}',z_i)^\top B_{\omega^{-}}(s_{i, k}', z_i)  \right)^2$\\ 
	 $\qquad\qquad\qquad - \frac{1}{I J} \sum_{i} \sum_{j} F_{\theta_k}(s_{i, j}, a_{i, j},z_i)^\top  B_\omega(s_{i, j}', z_i), \quad  \forall m \in [M]$
	
	\State {\bf  \textcolor{gray!50!blue}{/* \texttt{Compute orthonormality regularization loss}}}
	\State $\mathscr{L}_{\texttt{ortho}} (\omega) = \frac{1}{2 IJ (J-1)} \sum_i \sum_{j \neq k} (B_{\omega}(s_{i, j}', z_i)^\top B_{\omega}(s_{i, k}', z_i)^2 - \frac{1}{IJ} \sum_{i} \sum_{j} B_{\omega}(s_{i, j}', z_i)^\top B_{\omega}(s_{i, j}', z_i) $
	\State {\bf  \textcolor{gray!50!blue}{/* \texttt{Compute actor loss}}}
 \State $A(s_{i, j},a_{i,j},z_i) \leftarrow \sum_{m}F_{\theta_m}(s_{i,j}, a_{i,j}, z_i)^Tz_i - \E_{a_{i,j}'\sim \pi_\phi(s_{i,j}, z_i)}[\min_m F_{\theta_m}(s_{i, j}, a_{i,j}', z_i)^Tz_i]$
 \State $w(s_{i,j}, a_{i, j}, z_{i}) \leftarrow \frac{\exp(A(s_{i,j}, a_{i,j}, z_i)/\beta)}{\sum_{i',j'}\exp(A(s_{i', j'}, a_{i',j'}, z_{i'})/\beta)}$
 \State $w'(s_{i,j}, a_{i,j}, z_i)\propto \frac{w(s_{i,j}, a_{i,j},
    z_i)}{\sum_{(i',j') \neq (i,j)}w(s_{i', j'}, a_{i', j'}, z_{j'})} $
    \State $\mathscr{L}_{\texttt{actor}} (\phi) = -\frac{1}{IJ}\sum_{i, j} \ \overline{w'}(s_{i,j}, a_{i,j}, z_i)\log \pi_\phi(a_{i,j} \mid s_{i,j}, z_i) $
\State {\bf  \textcolor{gray!50!blue}{/* \texttt{Update all networks}}}
\State $\theta_m \leftarrow \theta_m - \xi \nabla_{\theta_m}  (\mathscr{L}_{\texttt{FB}}(\theta_k, \omega)$  for all $m\in[M]$
	\State $\omega \leftarrow \omega - \xi \nabla_{\omega}  (\sum_{l\in[m]} \mathscr{L}_{\texttt{FB}}(\theta_l, \omega) + \lambda \cdot  \mathscr{L}_{\texttt{ortho}}(\omega))$
	\State $\phi \leftarrow \phi - \xi \nabla_{\phi} \mathscr{L}_{\texttt{actor}} (\phi)$
	\State {\bf  \textcolor{gray!50!blue}{/* \texttt{Update target networks}}}
	\State  $\theta_m^{-} \leftarrow \zeta \theta_m^{-} + (1-\zeta) \theta_m$  for all $m\in[M]$ 
	\State  $\omega^{-} \leftarrow \zeta \omega^{-} + (1-\zeta) \omega$
 \EndFor
 \end{algorithmic}
 \end{small}
	\end{algorithm}

\section{Experimental Details}

\subsection{Network Architecture}

\begin{itemize}
    \item 
    
    For the backward representation network $B(s, z)$, we first
    preprocess separately $s$ and $(s, z)$. For the preprocessing of $s$,
    we use a feedforward neural network with one single hidden layer of
    256 units. For $(s, z)$ preprocessing, we use a masked network with one single hidden dimension of 256 units. The masked network employs multiplicative binary masks to remove some connections, such that each output layer unit of an autoregressive block is only predicted from the input units of previous blocks.
    After preprocessing, we concatenate the two outputs and pass them into a two hidden layer masked network that outputs a
    $d$-dimensional embedding. 

    \item  For the forward network $F(s,a,z)$, we first preprocess
    separately $(s, a)$ and $(s, z)$ by two feedforward networks with one
    single hidden layer (with 1024 units), and a 512-dimentional output space.
    Then we concatenate these two outputs and pass it into a
    three-hidden-layer feedforward network (with 1024 units) to output a
    $d$-dimensional vector. We use an ensemble of two networks for $F$.
    
    \item For the policy network $\pi(s,z)$, we first preprocess
    separately $s$ and $(s, z)$ by two feedforward networks with one
    single hidden layer (with 1024 units) into a 512-dimentional output space.
    Then we concatenate these two outputs and pass it into another four hidden layer feedforward network (with 1024 units)
    to output a $d_{A}$-dimensional vector. Then we apply a \texttt{Tanh}
    activation as the action space is $[-1,1]^{d_{A}}$.
    \end{itemize}

For all the architectures, we apply a layer normalization and \texttt{Tanh} activation in the first layer in order to standardize the states and actions. We use $\texttt{Relu}$ for the rest of layers.

\subsection{Extended DeepMind Control Tasks}
\label{sec:newtasks}

To evaluate our new unsupervised RL algorithm across a set of diverse unseen problems, we extend the DeepMind Control suite \edo{cit} with 15 new unseen tasks as defined by the following objectives:
\begin{itemize}
\item \textbf{cheetah bounce} -- Simulated cheetah agent is rewarded for advancing while elevating its trunk and maximizing vertical velocity.
\item \textbf{cheetah march} -- Simulated cheetah agent is rewarded for advancing at a constant pace.
\item \textbf{cheetah stand} -- Simulated cheetah agent is rewarded for standing upright on its back leg.
\item \textbf{cheetah headstand} -- Simulated cheetah agent is rewarded for standing on its head while raising its back leg.
\item \textbf{quadruped bounce} -- Simulated quadruped agent is rewarded for advancing while elevating its trunk and maximizing vertical velocity.
\item \textbf{quadruped skip} -- Simulated quadruped agent is rewarded for moving in a diagonal pattern across the environment.
\item \textbf{quadruped march} -- Simulated quadruped agent is rewarded for advancing at a constant pace.
\item \textbf{quadruped trot} -- Simulated quadruped agent is rewarded for advancing while minimizing feet contact with the ground.
\item \textbf{walker flip} -- Simulated walker for performing a cartwheel, flipping its body 360 degrees.
\item \textbf{walker march} -- Simulated walker agent is rewarded for advancing at a constant pace.
\item \textbf{walker skyreach} -- Simulated walker agent is rewarded for pushing either of its legs to maximize vertical reach.
\item \textbf{walker pullup} -- Simulated walker agent is rewarded for pushing its upper trunk vertically while keeping its feet firmly grounded.
\item \textbf{humanoid dive} -- Simulated humanoid agent is rewarded for diving head-first to maximize vertical velocity.
\item \textbf{humanoid march} -- Simulated humanoid agent is rewarded for advancing at a constant pace.
\item \textbf{humanoid skip} -- Simulated humanoid agent is rewarded for moving in a diagonal pattern across the environment.
\end{itemize}
We hope this new set of problems might facilitate the evaluation of simulated robotics agents for the broader RL field, even beyond the unsupervised setting.

\section{Full Tables of Results}
\label{sec:full_results}

\subsection{Jaco Arm Results}

\begin{table}[h]
    \centering
    \resizebox{1\columnwidth}{!}{
\begin{tabular}{|l|c|c|c|c|c|c|}
\hline
 Domain   & Task               & FB             & FB-AW             & FB-AWARE (4)         & FB-AWARE (8)      & LAP-AW        \\
\hline
 jaco     & reach\_bottom\_left  & 49.0\tiny{±25.5} & 43.9\tiny{±8.6}  & 56.3\tiny{±8.6}  & 63.6\tiny{±6.4}  & 25.9\tiny{±5.9}  \\
 jaco     & reach\_bottom\_right & 30.8\tiny{±7.5}  & 71.5\tiny{±18.2} & 57.6\tiny{±16.5} & 58.6\tiny{±21.1} & 34.0\tiny{±13.9} \\
 jaco     & reach\_random1      & 18.0\tiny{±8.0}  & 42.9\tiny{±15.7} & 64.4\tiny{±17.0} & 63.9\tiny{±10.7} & 20.4\tiny{±12.6} \\
 jaco     & reach\_random2      & 23.4\tiny{±6.4}  & 55.5\tiny{±5.6}  & 72.8\tiny{±10.7} & 63.7\tiny{±8.1}  & 14.3\tiny{±5.1}  \\
 jaco     & reach\_random3      & 43.2\tiny{±27.7} & 39.6\tiny{±5.9}  & 53.1\tiny{±6.4}  & 59.0\tiny{±12.1} & 14.6\tiny{±5.6}  \\
 jaco     & reach\_random4      & 32.6\tiny{±23.3} & 57.4\tiny{±11.5} & 68.4\tiny{±11.0} & 69.9\tiny{±10.3} & 24.1\tiny{±2.8}  \\
 jaco     & reach\_top\_left     & 32.6\tiny{±12.3} & 41.0\tiny{±5.4}  & 41.9\tiny{±8.3}  & 62.7\tiny{±14.9} & 10.3\tiny{±2.2}  \\
 jaco     & reach\_top\_right    & 21.5\tiny{±11.6} & 25.9\tiny{±9.2}  & 43.6\tiny{±9.7}  & 48.3\tiny{±12.1} & 21.4\tiny{±5.2}  \\
 \hline \hline
 jaco     & Average                & 31.4\tiny{±15.3} & 47.2\tiny{±10.0} & 57.3\tiny{±11.0} & 61.2\tiny{±12.0} & 20.6\tiny{±6.7}  \\
\hline
\end{tabular}
}
\caption{JACO results on RND dataset, with dimension $d=64$}
    \label{tab:jaco_rnd_64}
\end{table}

\begin{table}[]
    \centering
    \resizebox{1\columnwidth}{!}{
\begin{tabular}{|l|l|l|l|l|l|l|}
\hline
 Domain   & Task                & FB             & FB-AW             & FB-AWARE (4)         & FB-AWARE (8)      & LAP-AW            \\
\hline
 jaco     & reach\_bottom\_left  & 33.8\tiny{±17.3} & 76.0\tiny{±12.0} & 88.1\tiny{±18.5}  & 64.6\tiny{±13.4} & 41.3\tiny{±10.2} \\
 jaco     & reach\_bottom\_right & 51.3\tiny{±10.2} & 86.3\tiny{±9.4}  & 87.7\tiny{±15.0}  & 96.8\tiny{±6.9}  & 47.8\tiny{±18.1} \\
 jaco     & reach\_random1      & 32.6\tiny{±18.1} & 75.3\tiny{±10.5} & 85.4\tiny{±9.2}   & 87.0\tiny{±13.0} & 30.9\tiny{±5.3}  \\
 jaco     & reach\_random2      & 22.9\tiny{±10.0} & 86.3\tiny{±9.1}  & 104.1\tiny{±7.5}  & 95.5\tiny{±12.7} & 28.8\tiny{±6.3}  \\
 jaco     & reach\_random3      & 31.2\tiny{±9.0}  & 68.3\tiny{±11.3} & 89.5\tiny{±17.7}  & 61.9\tiny{±8.2}  & 24.7\tiny{±7.2}  \\
 jaco     & reach\_random4      & 21.6\tiny{±6.3}  & 82.2\tiny{±9.6}  & 101.2\tiny{±17.6} & 82.9\tiny{±10.5} & 34.7\tiny{±10.9} \\
 jaco     & reach\_top\_left     & 44.4\tiny{±16.6} & 59.5\tiny{±18.3} & 56.5\tiny{±9.6}   & 46.0\tiny{±17.7} & 32.1\tiny{±10.2} \\
 jaco     & reach\_top\_right    & 28.3\tiny{±13.0} & 44.2\tiny{±12.5} & 47.5\tiny{±5.6}   & 39.7\tiny{±8.4}  & 23.1\tiny{±7.1}  \\
 \hline \hline
 jaco     & Average                & 33.3\tiny{±12.6} & 72.2\tiny{±11.6} & 82.5\tiny{±12.6}  & 71.8\tiny{±11.4} & 32.9\tiny{±9.4}  \\
\hline
\end{tabular}
}
\caption{JACO results on RND dataset, with dimension $d=128$}
    \label{tab:jaco_rnd_128}
\end{table}

\begin{table}[]
    \centering
    \resizebox{.8\columnwidth}{!}{
\begin{tabular}{|l|l|l|l|l|l|l|}
\hline
 Domain   & Task                & FB             & FB-AW             & FB-AWARE (4)         & FB-AWARE (8)      & LAP-AW            \\
\hline
 jaco     & reach\_bottom\_left  & 33.1\tiny{±23.3} & 60.2\tiny{±26.8} & 101.0\tiny{±7.5} & 91.5\tiny{±16.8} & 42.7\tiny{±15.0} \\
 jaco     & reach\_bottom\_right & 10.5\tiny{±2.4}  & 69.2\tiny{±32.5} & 116.9\tiny{±9.7} & 90.9\tiny{±12.9} & 63.9\tiny{±13.8} \\
 jaco     & reach\_random1      & 31.9\tiny{±21.8} & 49.5\tiny{±25.5} & 86.0\tiny{±19.2} & 84.7\tiny{±12.4} & 45.4\tiny{±4.8}  \\
 jaco     & reach\_random2      & 32.1\tiny{±16.4} & 50.9\tiny{±26.2} & 99.7\tiny{±17.6} & 94.9\tiny{±14.5} & 47.2\tiny{±8.4}  \\
 jaco     & reach\_random3      & 43.6\tiny{±18.4} & 39.6\tiny{±22.4} & 72.4\tiny{±13.6} & 81.5\tiny{±9.9}  & 44.9\tiny{±9.4}  \\
 jaco     & reach\_random4      & 32.8\tiny{±15.4} & 58.3\tiny{±34.5} & 98.8\tiny{±15.7} & 93.1\tiny{±14.2} & 51.3\tiny{±8.2}  \\
 jaco     & reach\_top\_left     & 29.4\tiny{±11.2} & 38.4\tiny{±18.6} & 43.6\tiny{±10.3} & 65.1\tiny{±21.1} & 43.0\tiny{±6.1}  \\
 jaco     & reach\_top\_right    & 27.7\tiny{±6.2}  & 29.8\tiny{±19.2} & 44.8\tiny{±20.1} & 45.4\tiny{±13.8} & 37.0\tiny{±9.4}  \\
 \hline \hline
 jaco     & Average                & 30.1\tiny{±14.4} & 49.5\tiny{±25.7} & 82.9\tiny{±14.2} & 80.9\tiny{±14.4} & 46.9\tiny{±9.4}  \\
\hline
\end{tabular}
}
\caption{JACO results on RND dataset, with dimension $d=256$}
    \label{tab:jaco_rnd_256}
\end{table}

\subsection{DMC Locomotion Results}

\begin{table}[]
    \centering
    \resizebox{1\columnwidth}{!}{
    \begin{tabular}{|l|l|l|l|l|l|l|l|}
\hline
  Domain    & Task           & LAP&LAP-AW& FB-AW               & FB               & FB-ARE (4)           & FB-ARE (8)                \\
\hline
 cheetah   & walk           & 641.0\tiny{±137.7}&528.9\tiny{±22.9}& 520.8\tiny{±56.6}  & 780.3\tiny{±182.7} & 737.5\tiny{±204.7} & 686.1\tiny{±44.8}  \\
 cheetah   & run            & 156.5\tiny{±36.8}&116.0\tiny{±6.8}& 141.2\tiny{±19.8}  & 306.8\tiny{±87.3}  & 261.1\tiny{±95.3}  & 241.8\tiny{±57.3}  \\
 cheetah   & walk\_backward  & 930.9\tiny{±77.9}&452.3\tiny{±212.6}& 839.5\tiny{±49.3}  & 732.5\tiny{±167.3} & 769.9\tiny{±209.9} & 762.3\tiny{±205.6} \\
 cheetah   & run\_backward   & 230.7\tiny{±42.6}&109.7\tiny{±37.2}& 196.2\tiny{±22.4}  & 136.5\tiny{±35.6}  & 174.1\tiny{±48.6}  & 186.4\tiny{±66.4}  \\
 \hline \hline
 cheetah   & in\_dataset\_avg      & 489.8\tiny{±73.8}&301.7\tiny{±69.9}& 424.4\tiny{±37.0}  & 489.0\tiny{±118.2} & 485.7\tiny{±139.6} & 469.2\tiny{±93.5}  \\
 \hline \hline
 quadruped & walk           & 509.1\tiny{±38.9}&418.3\tiny{±42.7}& 438.2\tiny{±192.6} & 608.4\tiny{±72.0}  & 630.1\tiny{±96.9}  & 604.0\tiny{±116.8} \\
 quadruped & run            & 457.6\tiny{±27.7}&355.6\tiny{±86.5}& 391.2\tiny{±91.9}  & 392.7\tiny{±31.4}  & 417.4\tiny{±30.6}  & 376.1\tiny{±29.2}  \\
 quadruped & stand          & 681.6\tiny{±221.6}&731.3\tiny{±166.6}& 762.2\tiny{±152.8} & 687.9\tiny{±29.6}  & 761.7\tiny{±75.9}  & 705.4\tiny{±58.1}  \\
 quadruped & jump           & 464.5\tiny{±167.3}&493.4\tiny{±147.1}& 563.7\tiny{±139.1} & 567.0\tiny{±10.6}  & 609.2\tiny{±42.3}  & 580.3\tiny{±37.1}  \\
 \hline \hline
 quadruped & in\_dataset\_avg      & 528.2\tiny{±113.9}&499.6\tiny{±110.7}& 538.9\tiny{±144.1} & 564.0\tiny{±35.9}  & 604.6\tiny{±61.4}  & 566.4\tiny{±60.3}  \\
 \hline \hline
 walker    & stand          & 963.6\tiny{±15.3}&803.3\tiny{±61.9}& 452.6\tiny{±85.8}  & 728.5\tiny{±83.0}  & 632.7\tiny{±151.7} & 516.1\tiny{±191.0} \\
 walker    & walk           & 908.8\tiny{±28.1}&605.3\tiny{±36.4}& 572.0\tiny{±25.3}  & 669.9\tiny{±46.6}  & 607.9\tiny{±140.2} & 552.2\tiny{±268.4} \\
 walker    & run            & 318.7\tiny{±15.0}&196.6\tiny{±13.1}& 181.2\tiny{±16.5}  & 356.2\tiny{±20.9}  & 290.4\tiny{±22.7}  & 240.0\tiny{±122.7} \\
 walker    & spin           & 982.9\tiny{±3.5}&627.9\tiny{±135.1}& 963.7\tiny{±5.3}   & 974.9\tiny{±10.0}  & 983.4\tiny{±1.3}   & 788.2\tiny{±391.2} \\
 \hline \hline
 walker    & in\_dataset\_avg      & 793.5\tiny{±15.5}&558.3\tiny{±61.6}& 542.4\tiny{±33.2}  & 682.4\tiny{±40.1}  & 628.6\tiny{±79.0}  & 524.1\tiny{±243.3} \\
 \hline \hline
 cheetah   & bounce         & 600.4\tiny{±23.0}&428.2\tiny{±210.5}& 539.9\tiny{±25.8}  & 415.7\tiny{±119.2} & 472.8\tiny{±44.2}  & 462.2\tiny{±23.8}  \\
 cheetah   & march          & 290.8\tiny{±63.2}&233.6\tiny{±12.9}& 279.2\tiny{±39.0}  & 561.4\tiny{±183.7} & 531.5\tiny{±187.6} & 460.9\tiny{±119.2} \\
 cheetah   & stand          & 790.1\tiny{±107.9}&249.5\tiny{±131.6}& 738.7\tiny{±66.2}  & 780.9\tiny{±105.7} & 629.3\tiny{±49.3}  & 762.9\tiny{±179.9} \\
 cheetah   & headstand      & 577.9\tiny{±145.1}&288.9\tiny{±236.8}& 728.4\tiny{±83.3}  & 794.7\tiny{±9.4}   & 791.1\tiny{±52.0}  & 765.1\tiny{±70.3}  \\
 \hline \hline
 cheetah   & out\_of\_dataset\_avg       & 564.8\tiny{±84.8}&300.0\tiny{±147.9}& 571.6\tiny{±53.6}  & 638.2\tiny{±104.5} & 606.2\tiny{±83.3}  & 612.8\tiny{±98.3}  \\
 \hline \hline
 quadruped & bounce         & 179.5\tiny{±76.1}&123.2\tiny{±64.4}& 189.3\tiny{±192.0} & 276.1\tiny{±57.1}  & 196.0\tiny{±105.6} & 251.1\tiny{±47.9}  \\
 quadruped & skip           & 365.3\tiny{±108.4}&458.0\tiny{±114.5}& 559.5\tiny{±220.3} & 603.3\tiny{±30.7}  & 635.2\tiny{±34.4}  & 615.3\tiny{±35.2}  \\
 quadruped & march          & 478.7\tiny{±14.1}&370.4\tiny{±80.7}& 396.4\tiny{±123.7} & 458.3\tiny{±20.4}  & 466.0\tiny{±38.7}  & 419.7\tiny{±32.0}  \\
 quadruped & trot           & 310.6\tiny{±7.0}&246.7\tiny{±54.6}& 278.3\tiny{±122.4} & 357.6\tiny{±12.1}  & 380.9\tiny{±47.6}  & 335.3\tiny{±35.4}  \\
 \hline \hline
 quadruped & out\_of\_dataset\_avg       & 333.5\tiny{±51.4}&299.6\tiny{±78.5}& 355.9\tiny{±164.6} & 423.8\tiny{±30.1}  & 419.5\tiny{±56.6}  & 405.4\tiny{±37.6}  \\
 \hline \hline
 walker    & flip           & 605.4\tiny{±42.4}&435.0\tiny{±27.7}& 293.6\tiny{±83.8}  & 445.5\tiny{±77.4}  & 462.0\tiny{±68.7}  & 322.6\tiny{±147.7} \\
 walker    & march          & 695.8\tiny{±47.8}&364.3\tiny{±29.0}& 359.9\tiny{±15.2}  & 518.4\tiny{±95.9}  & 400.5\tiny{±178.0} & 390.3\tiny{±190.4} \\
 walker    & skyreach       & 653.8\tiny{±64.1}&423.1\tiny{±61.6}& 406.0\tiny{±13.8}  & 417.0\tiny{±34.9}  & 331.7\tiny{±44.4}  & 261.5\tiny{±151.7} \\
 walker    & pullup         & 264.8\tiny{±80.5}&58.4\tiny{±39.2}& 264.5\tiny{±60.5}  & 305.9\tiny{±109.4} & 463.1\tiny{±107.5} & 214.6\tiny{±155.1} \\
 \hline \hline
 walker    & out\_of\_dataset\_avg       & 554.9\tiny{±58.7}&320.2\tiny{±39.4}& 331.0\tiny{±43.3}  & 421.7\tiny{±79.4}  & 414.3\tiny{±99.7}  & 297.2\tiny{±161.2} \\
\hline
\end{tabular}
}
    \caption{DMC Locomotion results on RND dataset, with dimension 64,
    averaged over 100 episodes. Humanoid is not included, as RND produces insufficient
    exploration for Humanoid: even classical single-task (non-FB) training fails.
    }

    \label{tab:dmc_rnd}
\end{table}

\begin{table}[H]
    \centering
    \resizebox{1\columnwidth}{!}{
    \begin{tabular}{|l|l|l|l|l|l|l|}
\hline
 Domain    & Task          & FB               & FB-AW               & FB-AWARE (4)           & FB-AWARE (8)          & LAP-AW                \\
\hline
 cheetah   & walk          & 985.3\tiny{±3.1}   & 983.6\tiny{±6.6}   & 967.5\tiny{±28.4}  & 982.0\tiny{±4.2}  & 978.5\tiny{±14.0}  \\
 cheetah   & run           & 213.2\tiny{±123.5} & 560.2\tiny{±40.7}  & 525.7\tiny{±53.6}  & 547.6\tiny{±20.5} & 448.9\tiny{±222.9} \\
 cheetah   & walk\_backward & 971.0\tiny{±24.3}  & 979.7\tiny{±5.1}   & 984.1\tiny{±0.7}   & 985.1\tiny{±1.2}  & 982.8\tiny{±3.5}   \\
 cheetah   & run\_backward  & 302.5\tiny{±58.9}  & 473.9\tiny{±6.0}   & 454.0\tiny{±17.9}  & 465.8\tiny{±9.9}  & 413.8\tiny{±25.1}  \\
 \hline \hline
 cheetah   & in\_dataset\_avg     & 618.0\tiny{±52.4}  & 749.4\tiny{±14.6}  & 732.8\tiny{±25.1}  & 745.1\tiny{±8.9}  & 706.0\tiny{±66.4}  \\
 \hline \hline 
 quadruped & walk          & 389.5\tiny{±238.2} & 935.5\tiny{±6.6}   & 926.8\tiny{±4.2}   & 919.3\tiny{±8.4}  & 819.1\tiny{±132.3} \\
 quadruped & run           & 298.1\tiny{±105.4} & 580.8\tiny{±62.5}  & 606.2\tiny{±33.0}  & 566.0\tiny{±47.8} & 610.6\tiny{±89.3}  \\
 quadruped & stand         & 615.4\tiny{±191.0} & 941.1\tiny{±6.6}   & 947.9\tiny{±4.7}   & 940.3\tiny{±9.4}  & 911.5\tiny{±27.9}  \\
 quadruped & jump          & 429.6\tiny{±134.5} & 779.1\tiny{±48.1}  & 841.8\tiny{±8.2}   & 751.0\tiny{±75.9} & 782.8\tiny{±48.4}  \\
 \hline \hline 
 quadruped & in\_dataset\_avg     & 433.2\tiny{±167.3} & 809.1\tiny{±30.9}  & 830.7\tiny{±12.5}  & 794.2\tiny{±35.4} & 781.0\tiny{±74.5}  \\
 \hline \hline
 walker    & stand         & 744.0\tiny{±119.3} & 962.2\tiny{±14.7}  & 963.9\tiny{±3.7}   & 963.4\tiny{±4.3}  & 961.2\tiny{±8.2}   \\
 walker    & walk          & 780.0\tiny{±310.8} & 943.8\tiny{±20.8}  & 941.4\tiny{±7.4}   & 922.4\tiny{±16.4} & 934.3\tiny{±12.7}  \\
 walker    & run           & 422.5\tiny{±167.4} & 594.9\tiny{±12.1}  & 606.5\tiny{±6.2}   & 600.8\tiny{±37.8} & 518.9\tiny{±39.7}  \\
 walker    & spin          & 481.6\tiny{±226.3} & 894.7\tiny{±84.4}  & 820.9\tiny{±114.6} & 894.8\tiny{±63.6} & 802.0\tiny{±180.7} \\
 \hline \hline 
 walker    & in\_dataset\_avg     & 607.0\tiny{±205.9} & 848.9\tiny{±33.0}  & 833.2\tiny{±33.0}  & 845.3\tiny{±30.5} & 804.1\tiny{±60.4}  \\
 \hline \hline 
 humanoid  & walk          & 9.5\tiny{±11.8}    & 793.5\tiny{±16.1}  & 789.4\tiny{±18.0}  & 791.5\tiny{±7.3}  & 715.4\tiny{±35.1}  \\
 humanoid  & stand         & 7.7\tiny{±3.8}     & 720.0\tiny{±23.5}  & 728.3\tiny{±30.1}  & 711.6\tiny{±24.4} & 587.2\tiny{±34.0}  \\
 humanoid  & run           & 2.4\tiny{±1.6}     & 276.5\tiny{±11.0}  & 266.7\tiny{±4.8}   & 273.6\tiny{±5.2}  & 246.4\tiny{±9.7}   \\
 \hline \hline 
 humanoid  & in\_dataset\_avg     & 6.5\tiny{±5.7}     & 596.7\tiny{±16.9}  & 594.8\tiny{±17.6}  & 592.3\tiny{±12.3} & 516.3\tiny{±26.3}  \\
 \hline \hline
 cheetah   & bounce        & 351.8\tiny{±119.9} & 479.0\tiny{±22.3}  & 506.9\tiny{±18.9}  & 494.6\tiny{±9.4}  & 338.4\tiny{±38.3}  \\
 cheetah   & march         & 521.4\tiny{±161.8} & 897.8\tiny{±38.1}  & 903.6\tiny{±30.6}  & 921.9\tiny{±8.9}  & 819.7\tiny{±60.8}  \\
 cheetah   & stand         & 731.5\tiny{±248.4} & 419.1\tiny{±48.7}  & 472.7\tiny{±49.4}  & 548.7\tiny{±37.7} & 184.1\tiny{±44.9}  \\
 cheetah   & headstand     & 560.4\tiny{±185.9} & 849.7\tiny{±49.9}  & 848.0\tiny{±33.8}  & 806.0\tiny{±21.4} & 7.2\tiny{±11.4}    \\
 \hline \hline
 cheetah   & out\_of\_dataset\_avg      & 541.3\tiny{±179.0} & 661.4\tiny{±39.8}  & 682.8\tiny{±33.2}  & 692.8\tiny{±19.4} & 337.3\tiny{±38.9}  \\
 \hline \hline
 quadruped & bounce        & 114.7\tiny{±98.1}  & 181.2\tiny{±61.4}  & 284.2\tiny{±31.7}  & 223.7\tiny{±20.6} & 202.8\tiny{±52.1}  \\
 quadruped & skip          & 425.0\tiny{±139.2} & 654.6\tiny{±88.0}  & 835.2\tiny{±86.9}  & 705.4\tiny{±67.8} & 769.2\tiny{±94.6}  \\
 quadruped & march         & 304.1\tiny{±133.2} & 747.1\tiny{±94.2}  & 791.8\tiny{±30.6}  & 800.0\tiny{±7.9}  & 742.5\tiny{±123.0} \\
 quadruped & trot          & 228.3\tiny{±127.1} & 573.0\tiny{±46.9}  & 614.8\tiny{±15.9}  & 594.6\tiny{±6.1}  & 509.8\tiny{±73.1}  \\
 \hline \hline
 quadruped & out\_of\_dataset\_avg      & 268.0\tiny{±124.4} & 539.0\tiny{±72.6}  & 631.5\tiny{±41.3}  & 580.9\tiny{±25.6} & 556.1\tiny{±85.7}  \\
 \hline \hline
 walker    & flip          & 404.7\tiny{±234.4} & 909.9\tiny{±20.0}  & 913.6\tiny{±14.5}  & 914.5\tiny{±18.3} & 780.2\tiny{±66.2}  \\
 walker    & march         & 663.3\tiny{±245.6} & 826.7\tiny{±41.3}  & 841.0\tiny{±45.1}  & 811.8\tiny{±17.7} & 725.8\tiny{±25.7}  \\
 walker    & skyreach      & 396.5\tiny{±35.2}  & 365.2\tiny{±42.7}  & 366.7\tiny{±25.2}  & 404.3\tiny{±51.5} & 284.6\tiny{±52.7}  \\
 walker    & pullup        & 124.1\tiny{±101.7} & 334.4\tiny{±112.7} & 523.7\tiny{±31.4}  & 454.8\tiny{±46.9} & 113.9\tiny{±48.8}  \\
 \hline \hline
 walker    & out\_of\_dataset\_avg      & 397.2\tiny{±154.2} & 609.0\tiny{±54.2}  & 661.3\tiny{±29.1}  & 646.4\tiny{±33.6} & 476.1\tiny{±48.3}  \\
 \hline \hline
 humanoid  & dive          & 165.8\tiny{±7.9}   & 404.1\tiny{±11.1}  & 409.8\tiny{±17.9}  & 396.5\tiny{±16.1} & 242.6\tiny{±18.1}  \\
 humanoid  & march         & 6.3\tiny{±8.4}     & 669.6\tiny{±23.5}  & 659.5\tiny{±23.6}  & 661.2\tiny{±18.3} & 559.6\tiny{±53.1}  \\
 humanoid  & skip          & 2.1\tiny{±0.9}     & 233.7\tiny{±39.2}  & 234.9\tiny{±11.4}  & 221.4\tiny{±15.1} & 126.0\tiny{±18.2}  \\
 \hline \hline
 humanoid  & out\_of\_dataset\_avg      & 58.1\tiny{±5.7}    & 435.8\tiny{±24.6}  & 434.7\tiny{±17.6}  & 426.4\tiny{±16.5} & 309.4\tiny{±29.8}  \\
\hline
\end{tabular}
}
    \caption{DMC locomotion results on MOOD dataset, with dimension = 64 for walker,cheetah, quadruped, and dimension = 128 for humanoid, averaged over 100 episodes}
    \label{tab:dmc_mood}
\end{table}

\subsection{Additional reward prompts}

We demonstrate the adaptability of our FB-AWARE model on the DMC humanoid by showcasing its behavior in response to various reward functions. In Figure~\ref{fig:eq.frames}, we illustrate the agent's actions when prompted by the following reward functions:
\begin{itemize}

\item $\textsc{Left\_Hand}$: the task consists in raising the left hand while standing. Specifically, the reward function is defined as having a velocity close to zero (exponential term), having an upright torso, and maintaining the height of the left wrist above a certain threshold while keeping the height of the right wrist below a different threshold.

\begin{equation*}
\mathrm{R}_{\textsc{Left\_Hand}} = \exp( -(v_x^2 + v_y^2) ) * \mathrm{upright} * \mathbb{I}\{\mathrm{left\_wrist\_z} > 2\} * \mathbb{I}\{\mathrm{right\_wrist\_z} < 0.9 \}
\end{equation*}

\item $\textsc{Right\_Hand}$: the task consists in raising the right hand while standing. Specifically, the reward function is defined as having a velocity close to zero (exponential term), having an upright torso, and maintaining the height of the right wrist above a certain threshold while keeping the height of the left wrist below a different threshold.

\begin{equation*}
\mathrm{R}_{\textsc{Right\_Hand}} = \exp( -(v_x^2 + v_y^2) ) * \mathrm{upright} * \mathbb{I}\{\mathrm{left\_wrist\_z} < 0.9\} * \mathbb{I}\{\mathrm{right\_wrist\_z} > 0.9 \}
\end{equation*}

\item $\textsc{Walk\_Open\_Hand}$: the task consists in walking while keeping the two hands open. The reward function is defined as having a velocity above some threshold and the absolute distance between the $y$ coordinate of the left and right wrist above some threshold.

\begin{equation*}
\mathrm{R}_{\textsc{Walk\_Open\_Hand}} = \mathbb{I}\{v_x^2 + v_y^2 > 5 \} * \mathbb{I}\{ | \mathrm{left\_wrist\_y} - \mathrm{right\_wrist\_y} | >  1.2\} 
\end{equation*}

\item $\textsc{Split}$: the task consists in doing a split on the ground. The reward can be described as having a velocity close to zero, the height of the pelvis below some threshold and the absolute distance between the $y$ coordinate of the left and right ankle above some threshold.

\begin{equation*}
\mathrm{R}_{\textsc{Split}} = \exp(-(v_x^2 + v_y^2)) 
* \mathbb{I}\{ \mathrm{z\_pelvis} < 0.2 \}
* \mathbb{I}\{ \mathrm{left\_ankle\_y} 
- \mathrm{right\_ankle\_y} | > 0.5 \}
\end{equation*}

\end{itemize}

\begin{figure}[t]
    \centering
\begin{tikzpicture}
 \node (a) at (0,0) {
\includegraphics[width=0.12\linewidth]{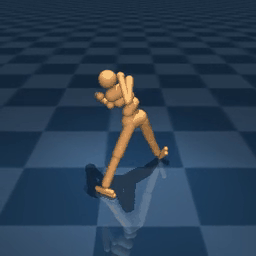}
\includegraphics[width=0.12\linewidth]{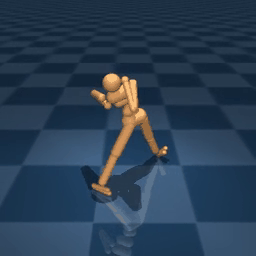}
\includegraphics[width=0.12\linewidth]{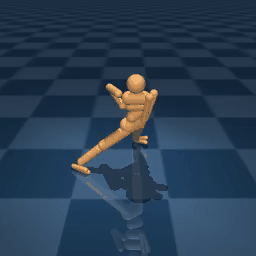}
\includegraphics[width=0.12\linewidth]{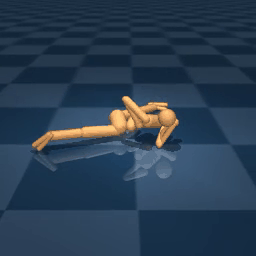}
\includegraphics[width=0.12\linewidth]{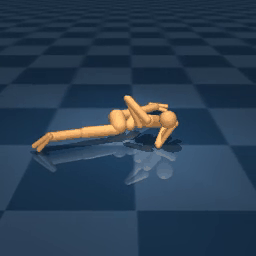}
\includegraphics[width=0.12\linewidth]{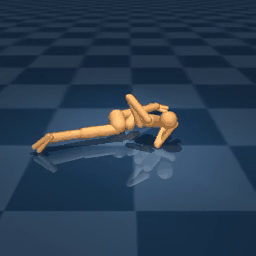}
\includegraphics[width=0.12\linewidth]{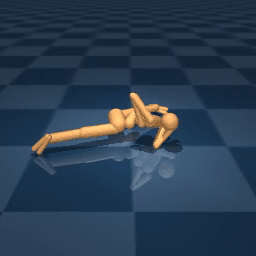}
};
 \node (b) at (0,-2) {
\includegraphics[width=0.12\linewidth]{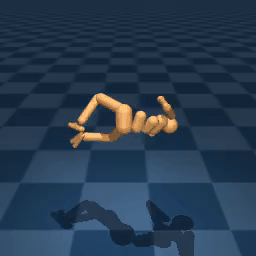}
\includegraphics[width=0.12\linewidth]{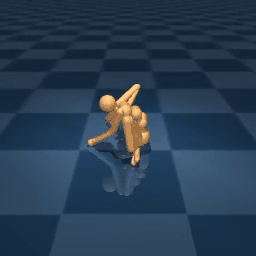}
\includegraphics[width=0.12\linewidth]{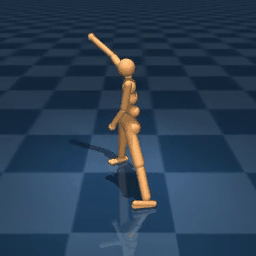}
\includegraphics[width=0.12\linewidth]{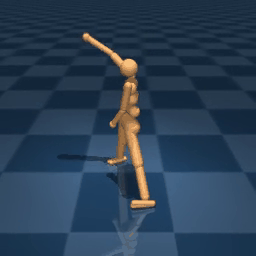}
\includegraphics[width=0.12\linewidth]{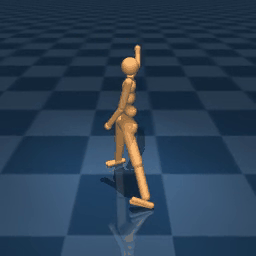}
\includegraphics[width=0.12\linewidth]{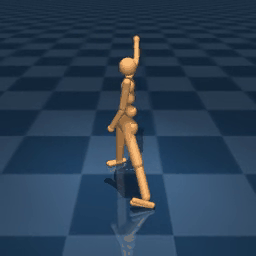}
\includegraphics[width=0.12\linewidth]{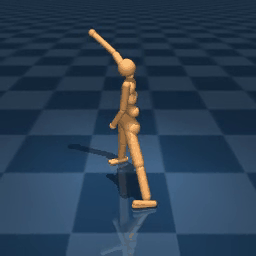}
};
 \node (c) at (0,-4) {
\includegraphics[width=0.12\linewidth]{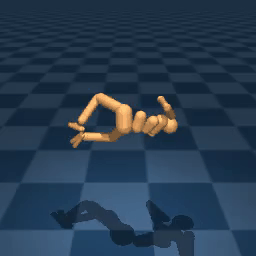}
\includegraphics[width=0.12\linewidth]{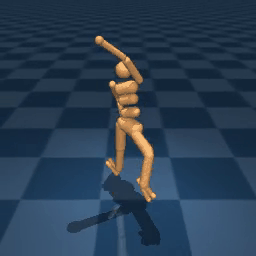}
\includegraphics[width=0.12\linewidth]{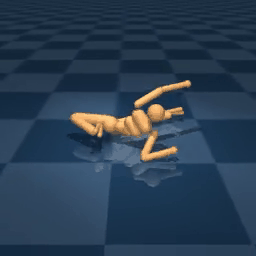}
\includegraphics[width=0.12\linewidth]{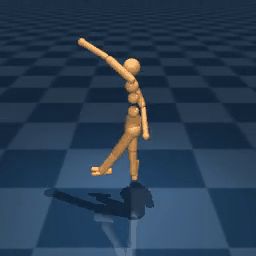}
\includegraphics[width=0.12\linewidth]{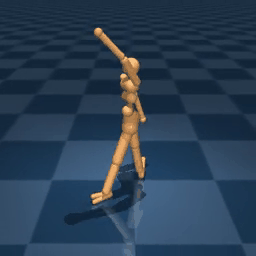}
\includegraphics[width=0.12\linewidth]{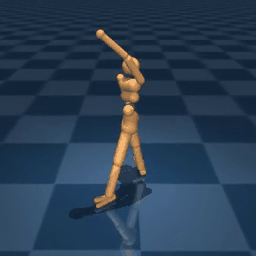}
\includegraphics[width=0.12\linewidth]{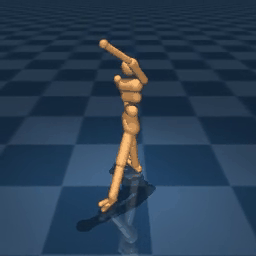}
};
 \node (d) at (0,-6) {
\includegraphics[width=0.12\linewidth]{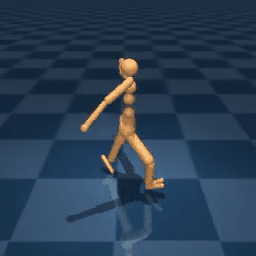}
\includegraphics[width=0.12\linewidth]{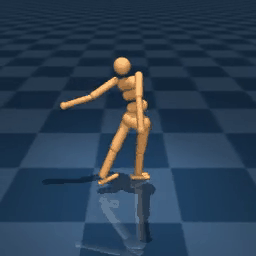}
\includegraphics[width=0.12\linewidth]{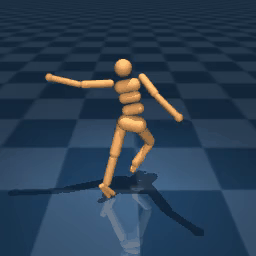}
\includegraphics[width=0.12\linewidth]{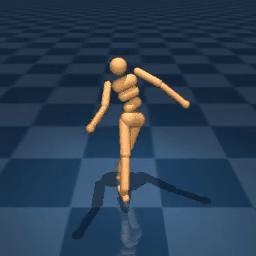}
\includegraphics[width=0.12\linewidth]{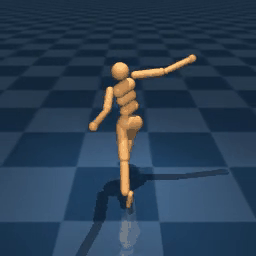}
\includegraphics[width=0.12\linewidth]{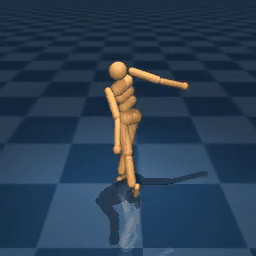}
\includegraphics[width=0.12\linewidth]{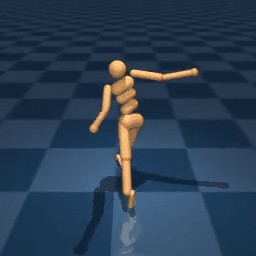}
};
  \node [left=1mm of a] {\rotatebox{90}{\scriptsize $\textsc{Split}$}};
  \node [left=1mm of b] {\rotatebox{90}{\scriptsize $\textsc{Right-Hand}$}};
  \node [left=1mm of c] {\rotatebox{90}{\scriptsize $\textsc{Left-Hand}$}};
  \node [left=1mm of d] {\rotatebox{90}{\scriptsize $\textsc{Walk-Open-Hand}$}};
  \node [above=0.5mm of a]  {\scriptsize Time $\longrightarrow$};
\end{tikzpicture}
    \caption{Example of behaviors inferred by from reward equations.}
    \label{fig:eq.frames}
\end{figure}

\section{Ablations}
\label{sec:ablations}

\begin{figure}[H]
    \centering
    \includegraphics[width=0.95\linewidth]{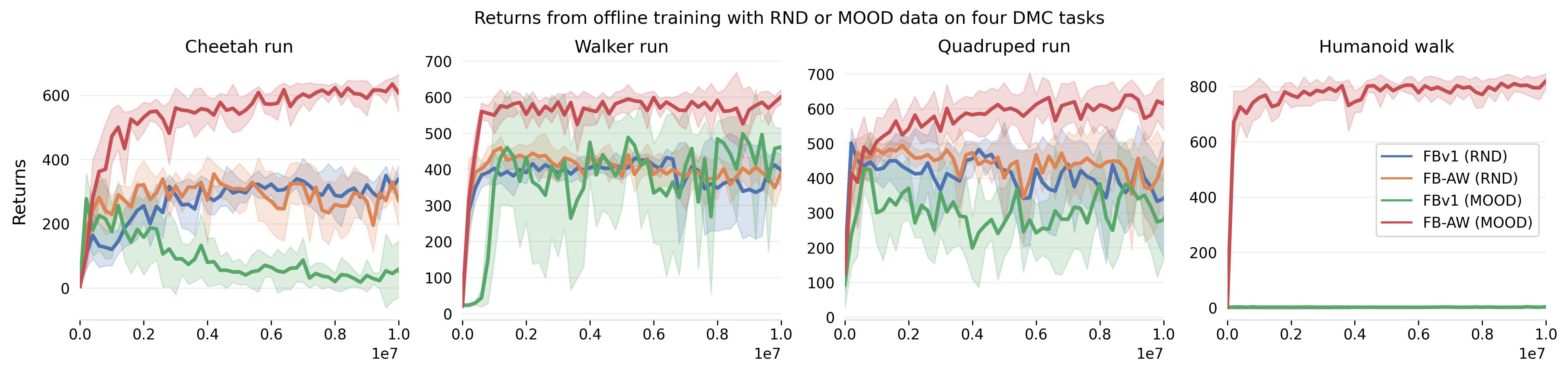}
    \caption{Performance on four representative tasks of the DMC when
    training FB-AW and vanilla FB (FBv1) either from mixed objective MOOD
    or pure RND data. The advantage of AW is clear on the mixed-objective
    MOOD datasets. The RND dataset does not allow FB to reach top
    performance. %
    }
    \label{sec4:fig:ex_perf_bias}
\end{figure}%

\begin{figure}[H]
    \centering
    \includegraphics[width=0.95\linewidth]{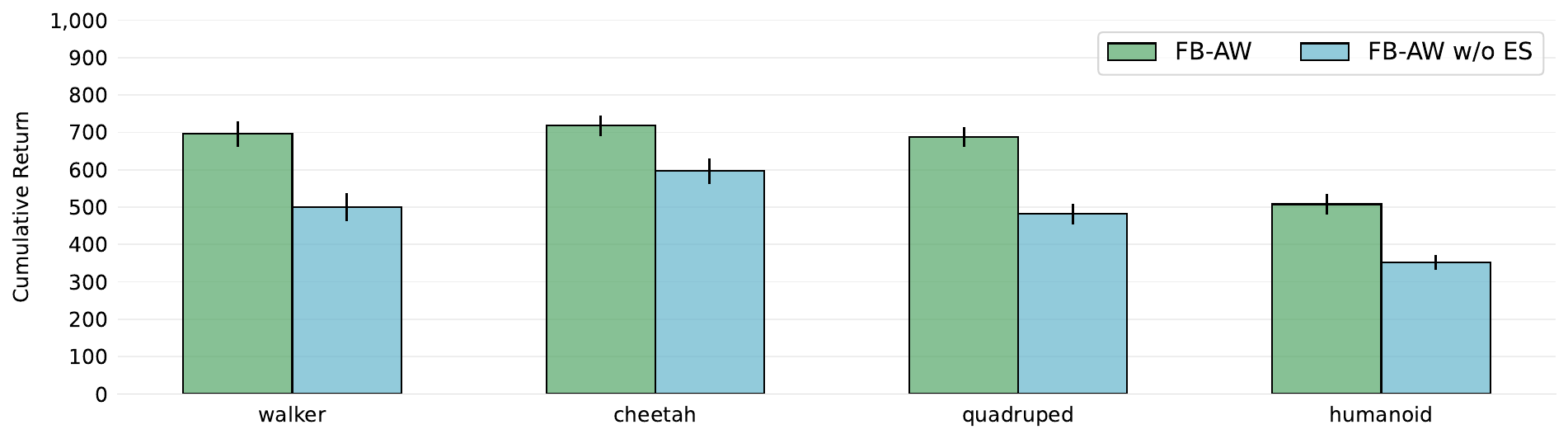}
    \caption{Cumulative reward averaged over all DMC Locomotion tasks,
    achieved by FB-AW with or without Evaluation-Sampling, trained on
    MOOD dataset.}
    \label{fig:es_abl}
\end{figure}%

\begin{figure}[H]
    \centering
    \includegraphics[width=0.95\linewidth]{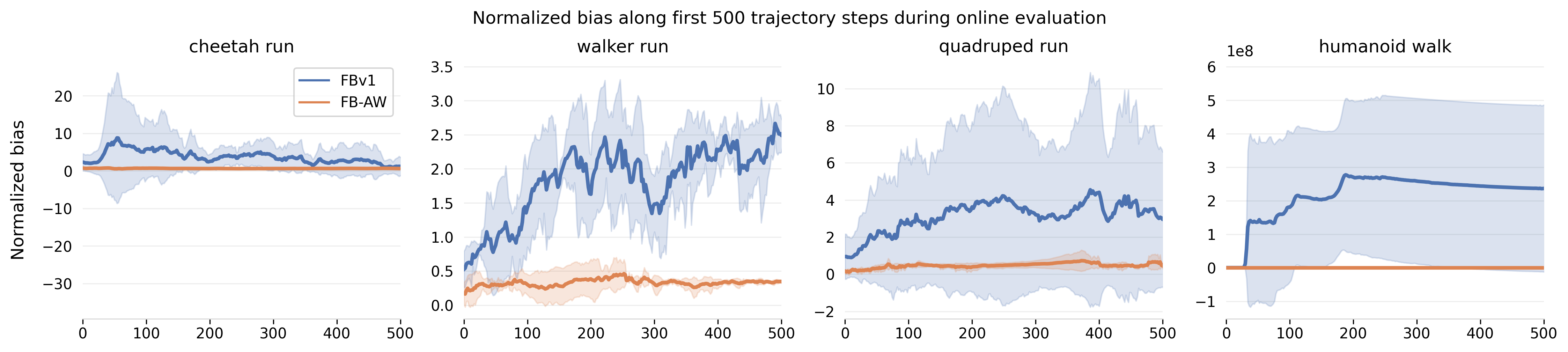}
    \caption{AW improves accuracy of reward prediction by the $B$ model.
    We plot the bias of estimated rewards when progressing through a
    trajectory, namely, the difference between the actual trajectory
    return and the return $\sum_t B(s_t)^T z$ predicted by FB, 
    after offline training on MOOD (averaged across 5 agents,
    10 trajectories each). Vanilla FB provides overoptimistic values.}
    \label{sec4:fig:ex_perf_bias}
\end{figure}%

\begin{figure}[H]
    \centering
    \includegraphics[width=0.95\linewidth]{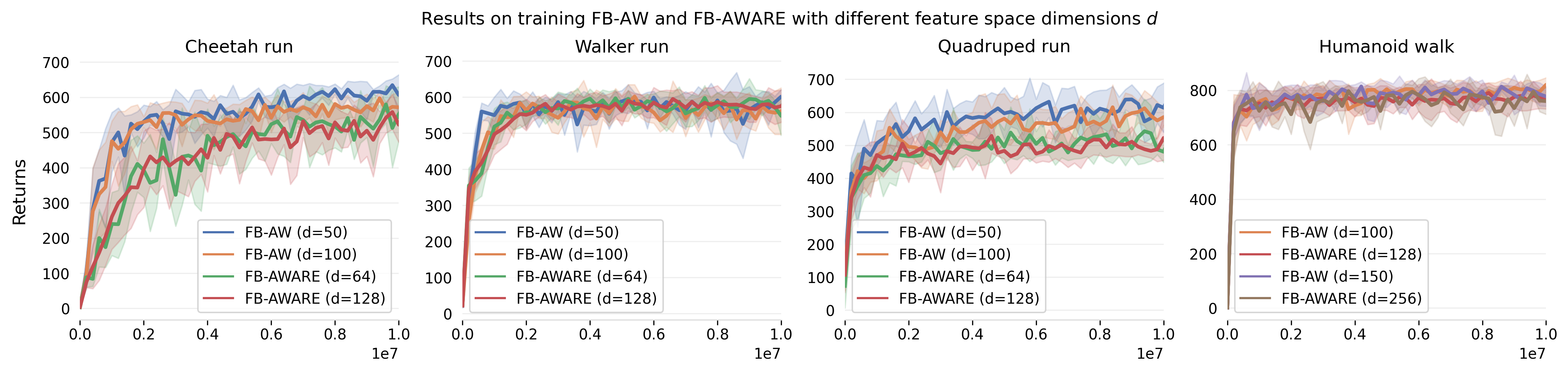}
    \caption{Effect of the $z$ dimension for FB-AW and FB-AWARE on the MOOD mixed objective datasets}
    \label{TODO}
\end{figure}%

\begin{figure}[H]
    \centering
    \includegraphics[width=0.95\linewidth]{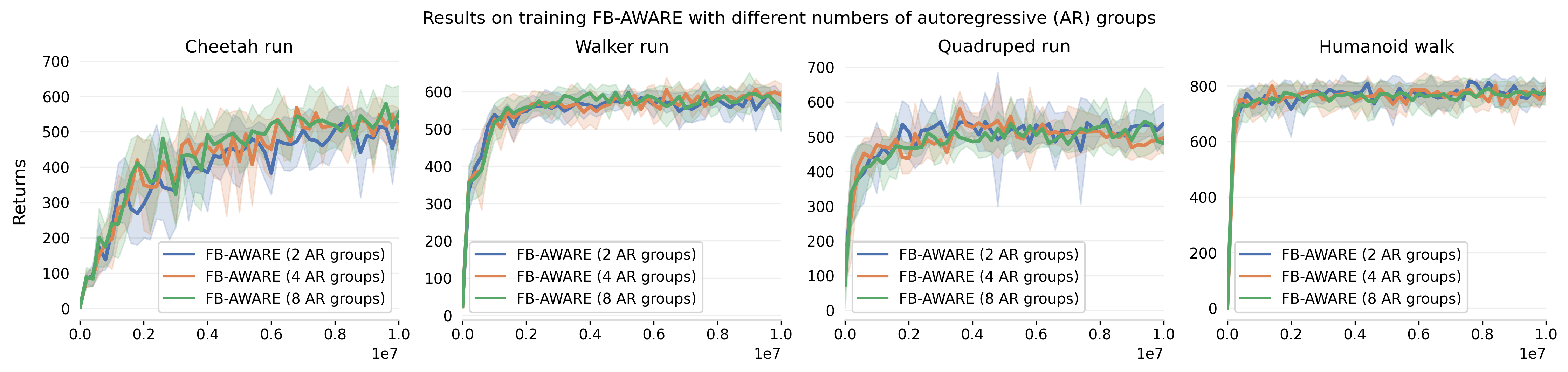}
    \caption{Effect of the number of autoregressive groups for FB-AWARE on the MOOD mixed objective datasets}
    \label{TODO}
\end{figure}%

\begin{figure}[H]
    \centering
    \includegraphics[width=0.95\linewidth]{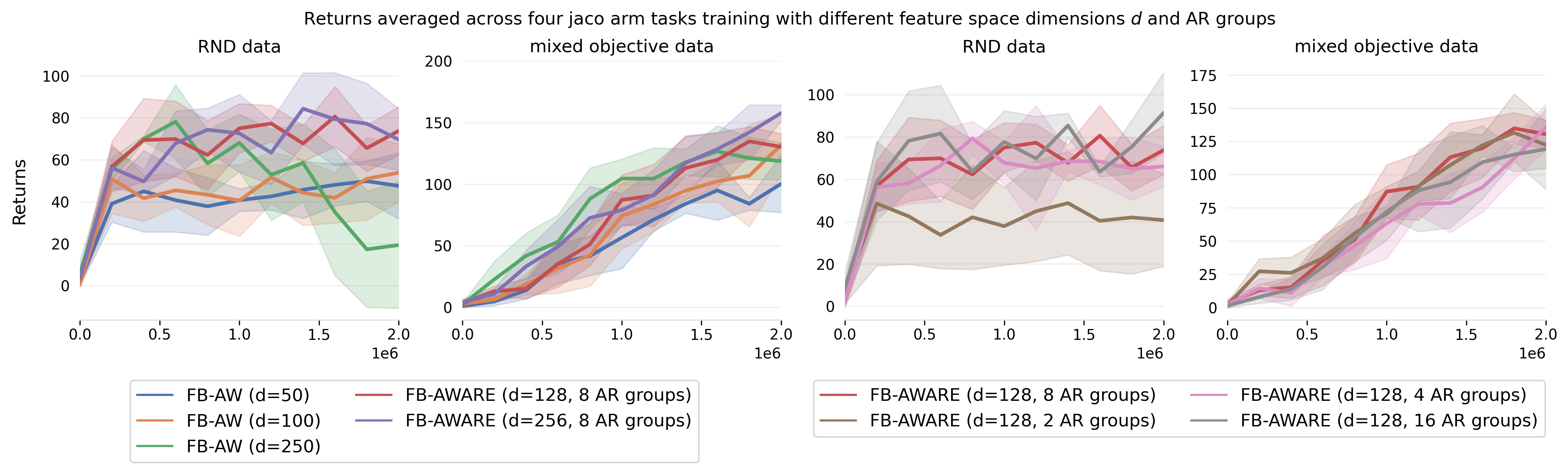}
    \caption{Effect of modifying the $z$ dimension and the number of
    autoregressive groups for FB-AW and FB-AWARE, for performance in the
    Jaco arm environment}
    \label{TODO}
\end{figure}%

\begin{figure}[H]
    \centering
    \includegraphics[width=0.95\linewidth]{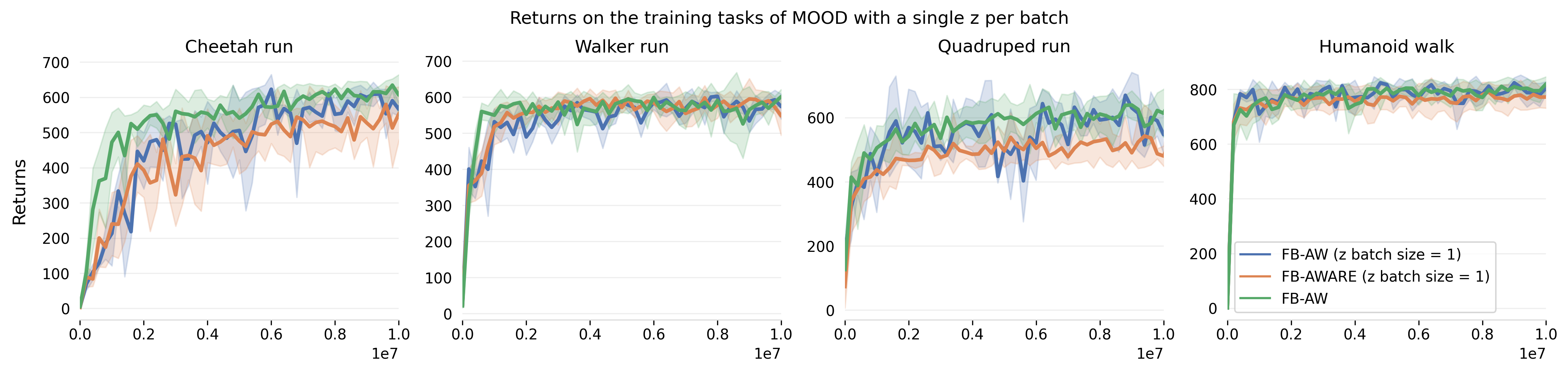}
    \caption{Effect of training FB-AW with a single $z$-per-batch like
    FB-AWARE (Section~\ref{sec:arfbalgo})}
    \label{fig:zperbatch}
\end{figure}%

\begin{figure}[H]
    \centering
    \includegraphics[width=0.95\linewidth]{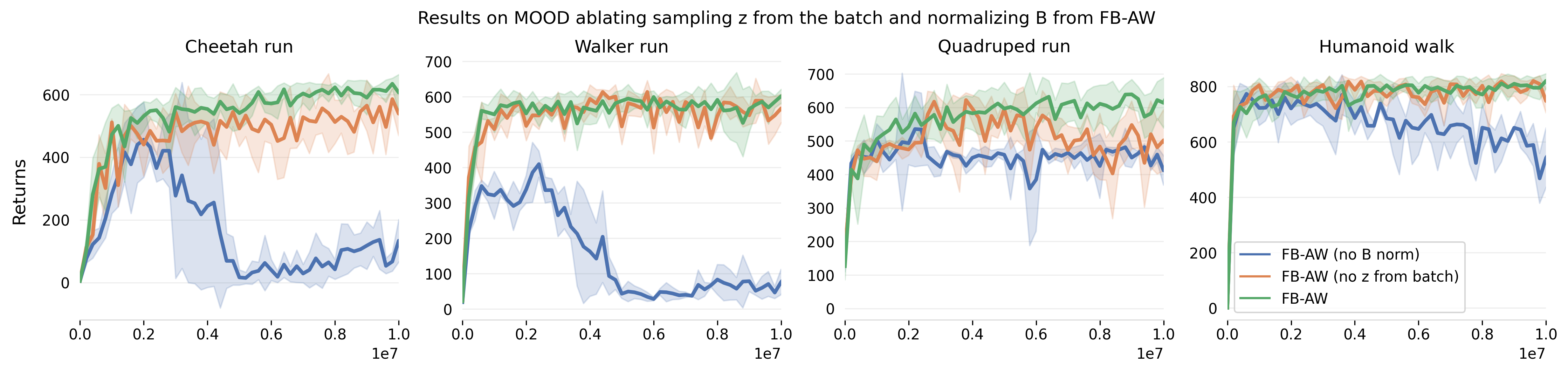}
        \includegraphics[width=0.95\linewidth]{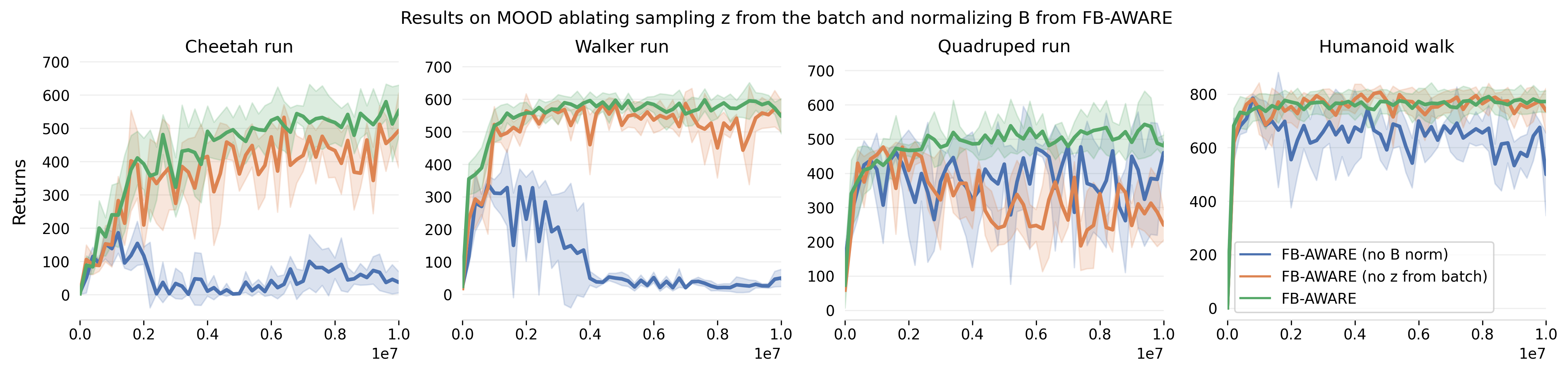}
    \caption{Ablations for Appendix~\ref{sec:arfbalgo}: Effects from
    training without the $B$ normalization \NDY{This is
    the normalization $B\gets B\sqrt{d}/\norm{B}$, right? or does this also
    remove the orthonormalization loss?} and
    without sampling 50\% of $z$ from other states in the minibatch
    for FB-AW (Top) and FB-AWARE (Bottom)}
    \label{fig:z_B_ablation}
\end{figure}%

\begin{table}[H]
\centering
\resizebox{1\columnwidth}{!}{
\begin{tabular}{|l|l|l|l|l|l|l|}
\hline
 Domain    & Task          &  Vanilla FB          &  FB \newline (fully par. + no min.)              & FB-AW \newline  (no fully par. + min.)    & FB-AW (WIS)      & FB-AW              \\
\hline
 cheetah   & walk          & 275.5\tiny{±343.5} & 500.2\tiny{±285.2} & 987.5\tiny{±0.6}    & 975.0\tiny{±23.5}  & 975.6\tiny{±24.6}  \\
 cheetah   & run           & 77.1\tiny{±87.3}   & 106.0\tiny{±50.4}  & 548.3\tiny{±38.6}   & 548.4\tiny{±30.0}  & 610.7\tiny{±49.2}  \\
 cheetah   & walk\_backward & 199.3\tiny{±209.1} & 349.7\tiny{±394.5} & 984.8\tiny{±0.4}    & 984.5\tiny{±1.2}   & 984.7\tiny{±0.4}   \\
 cheetah   & run\_backward  & 49.1\tiny{±58.6}   & 78.6\tiny{±118.3}  & 461.0\tiny{±6.6}    & 477.4\tiny{±7.8}   & 481.4\tiny{±4.2}   \\
 \hline \hline
 cheetah   & in\_dataset\_avg     & 0.175\tiny{±0.203} & 0.295\tiny{±0.249} & 0.939\tiny{±0.019}  & 0.944\tiny{±0.022} & 0.971\tiny{±0.028} \\
 \hline \hline
 quadruped & walk          & 318.3\tiny{±65.7}  & 763.9\tiny{±145.2} & 889.4\tiny{±80.5}   & 938.8\tiny{±34.7}  & 958.0\tiny{±9.1}   \\
 quadruped & run           & 279.4\tiny{±45.8}  & 439.8\tiny{±90.2}  & 446.1\tiny{±50.6}   & 599.4\tiny{±70.4}  & 673.4\tiny{±33.1}  \\
 quadruped & stand         & 619.6\tiny{±121.1} & 831.5\tiny{±147.8} & 901.3\tiny{±84.7}   & 944.9\tiny{±44.0}  & 975.9\tiny{±4.5}   \\
 quadruped & jump          & 435.0\tiny{±79.8}  & 644.1\tiny{±127.7} & 708.2\tiny{±20.2}   & 767.4\tiny{±74.9}  & 798.4\tiny{±55.0}  \\
 \hline \hline
 quadruped & in\_dataset\_avg     & 0.493\tiny{±0.093} & 0.800\tiny{±0.153} & 0.877\tiny{±0.070}  & 0.976\tiny{±0.070} & 1.026\tiny{±0.032} \\
 \hline \hline
 walker    & stand         & 906.5\tiny{±80.4}  & 976.2\tiny{±9.5}   & 964.8\tiny{±3.5}    & 978.9\tiny{±3.1}   & 956.1\tiny{±34.7}  \\
 walker    & walk          & 892.8\tiny{±102.7} & 939.9\tiny{±44.7}  & 946.7\tiny{±10.3}   & 960.1\tiny{±7.3}   & 955.4\tiny{±16.6}  \\
 walker    & run           & 462.2\tiny{±37.5}  & 487.0\tiny{±44.1}  & 480.5\tiny{±52.6}   & 583.3\tiny{±20.9}  & 579.8\tiny{±45.8}  \\
 walker    & spin          & 422.1\tiny{±121.4} & 464.9\tiny{±196.4} & 923.2\tiny{±30.1}   & 759.2\tiny{±173.6} & 789.6\tiny{±117.7} \\
 \hline \hline
 walker    & in\_dataset\_avg     & 0.749\tiny{±0.093} & 0.799\tiny{±0.081} & 0.913\tiny{±0.032}  & 0.918\tiny{±0.055} & 0.917\tiny{±0.061} \\
 \hline \hline
 humanoid  & walk          & 3.6\tiny{±5.0}     & 2.3\tiny{±1.1}     & 677.6\tiny{±52.8}   & 779.9\tiny{±37.1}  & 785.0\tiny{±20.4}  \\
 humanoid  & stand         & 5.1\tiny{±1.9}     & 4.8\tiny{±0.8}     & 481.3\tiny{±59.5}   & 750.3\tiny{±43.3}  & 801.7\tiny{±45.8}  \\
 humanoid  & run           & 1.1\tiny{±0.8}     & 1.1\tiny{±0.6}     & 256.8\tiny{±22.9}   & 274.4\tiny{±20.4}  & 294.9\tiny{±17.5}  \\
 \hline \hline
 humanoid  & in\_dataset\_avg     & 0.005\tiny{±0.004} & 0.004\tiny{±0.002} & 0.793\tiny{±0.074}  & 0.965\tiny{±0.058} & 1.014\tiny{±0.049} \\
 \hline \hline
 cheetah   & bounce        & 109.6\tiny{±83.9}  & 315.4\tiny{±102.8} & 436.3\tiny{±47.1}   & 469.3\tiny{±39.7}  & 502.5\tiny{±19.2}  \\
 cheetah   & march         & 135.8\tiny{±182.5} & 257.8\tiny{±110.2} & 766.9\tiny{±92.7}   & 892.6\tiny{±35.0}  & 917.4\tiny{±13.4}  \\
 cheetah   & stand         & 206.6\tiny{±383.7} & 684.7\tiny{±267.4} & 426.9\tiny{±191.6}  & 288.9\tiny{±162.6} & 387.6\tiny{±98.4}  \\
 cheetah   & headstand     & 233.8\tiny{±369.0} & 923.7\tiny{±57.1}  & 488.0\tiny{±345.1}  & 407.0\tiny{±369.8} & 854.5\tiny{±42.0}  \\
 \hline \hline
 cheetah   & out\_dataset\_avg      & 0.214\tiny{±0.308} & 0.684\tiny{±0.166} & 0.660\tiny{±0.207}  & 0.642\tiny{±0.188} & 0.832\tiny{±0.053} \\
 \hline \hline
 quadruped & bounce        & 133.6\tiny{±82.8}  & 248.9\tiny{±41.0}  & 204.3\tiny{±61.0}   & 246.6\tiny{±31.3}  & 231.8\tiny{±18.4}  \\
 quadruped & skip          & 478.9\tiny{±113.2} & 601.0\tiny{±72.3}  & 554.0\tiny{±1.3}    & 722.7\tiny{±112.4} & 836.1\tiny{±104.4} \\
 quadruped & march         & 280.8\tiny{±73.1}  & 517.2\tiny{±108.0} & 478.8\tiny{±41.8}   & 802.1\tiny{±55.9}  & 860.2\tiny{±26.7}  \\
 quadruped & trot          & 178.3\tiny{±38.0}  & 381.5\tiny{±82.6}  & 412.1\tiny{±48.6}   & 605.3\tiny{±16.9}  & 615.9\tiny{±13.6}  \\
 \hline \hline
 quadruped & out\_dataset\_avg      & 0.524\tiny{±0.175} & 0.897\tiny{±0.159} & 0.834\tiny{±0.105}  & 1.187\tiny{±0.105} & 1.245\tiny{±0.075} \\
 \hline \hline
 walker    & flip          & 630.3\tiny{±111.7} & 744.4\tiny{±105.4} & 771.8\tiny{±104.5}  & 891.2\tiny{±24.9}  & 896.5\tiny{±41.2}  \\
 walker    & march         & 709.4\tiny{±182.1} & 744.6\tiny{±183.3} & 593.9\tiny{±12.6}   & 788.3\tiny{±63.7}  & 797.2\tiny{±36.8}  \\
 walker    & skyreach      & 321.7\tiny{±130.0} & 423.0\tiny{±70.3}  & 392.7\tiny{±15.1}   & 364.3\tiny{±48.0}  & 389.8\tiny{±6.5}   \\
 walker    & pullup        & 114.7\tiny{±98.5}  & 101.3\tiny{±104.1} & 33.9\tiny{±8.2}     & 309.7\tiny{±145.5} & 376.5\tiny{±219.5} \\
 \hline \hline
 walker    & out\_dataset\_avg      & 0.645\tiny{±0.196} & 0.747\tiny{±0.163} & 0.674\tiny{±0.052}  & 0.849\tiny{±0.101} & 0.889\tiny{±0.101} \\
 \hline \hline
 humanoid  & dive          & 159.2\tiny{±12.2}  & 166.6\tiny{±29.5}  & 357.1\tiny{±36.6}   & 387.9\tiny{±39.8}  & 404.0\tiny{±5.7}   \\
 humanoid  & march         & 2.9\tiny{±3.5}     & 2.9\tiny{±1.8}     & 479.6\tiny{±48.5}   & 623.1\tiny{±64.6}  & 645.9\tiny{±49.7}  \\
 humanoid  & skip          & 2.0\tiny{±1.2}     & 1.7\tiny{±0.4}     & 81.3\tiny{±26.7}    & 254.4\tiny{±23.8}  & 250.1\tiny{±24.0}  \\
 \hline \hline
 humanoid  & out\_dataset\_avg      & 0.112\tiny{±0.011} & 0.117\tiny{±0.021} & 0.536\tiny{±0.072}  & 0.788\tiny{±0.079} & 0.805\tiny{±0.049} \\
\hline
\end{tabular}
}
\caption{Ablations regarding the various components from
Section~\ref{sec:aw}: advantage weighting, using the average versus the
min of the two target networks for representing uncertainty, using
fully parallel architectures for the two target networks, and using
improved weighted importance sampling (IWIS) versus ordinary WIS.
As described in the text, FB-AW (right column) has advantage weighting, uses the
average instead of the min, has fully parallel architectures, and uses
IWIS. Vanilla FB (left column) has the opposite settings. We compare other
combinations in between.
We report performance on the mixed objective datasets from MOOD, on both
in-dataset and out-of-dataset tasks. The representation dimension is $d=50$ for Cheetah, Quadruped, Walker,
and $d=100$ for Humanoid.
}
\label{appB:tab:full_bench_res}
\end{table}

\end{document}